\documentclass{article}

\usepackage[final]{neurips_2025}

\usepackage{natbib}
\usepackage[utf8]{inputenc} 
\usepackage[T1]{fontenc}    
\usepackage{hyperref}       
\usepackage{url}            
\usepackage{booktabs}       
\usepackage{amsfonts}       
\usepackage{nicefrac}       
\usepackage{microtype}      
\usepackage{xcolor}         

\newcommand{\Adam}{{\texttt{Adam}}}
\newcommand{\AdamW}{{\texttt{AdamW}}}
\newcommand{\Signum}{{\texttt{Signum}}}
\newcommand{\SignSGD}{{\texttt{SignSGD}}}
\newcommand{\SGD}{{\texttt{SGD}}}
\newcommand{\RMSprop}{{\texttt{RMSprop}}}
\newcommand{\Muon}{{\texttt{Muon}}}
\newcommand{\Scion}{{\texttt{Scion}}}
\newcommand{\SOAP}{{\texttt{SOAP}}}
\newcommand{\Shampoo}{{\texttt{Shampoo}}}

\newcommand{\var}{\sigma^2}

\usepackage{amssymb,amsmath} 
\usepackage{amsthm}
\usepackage{amsbsy}
\usepackage{mathtools}
\usepackage{mdframed} 
\usepackage{thmtools}
\usepackage{thm-restate}
\definecolor{shadecolor}{gray}{0.90}
\declaretheoremstyle[
headfont=\normalfont\bfseries,
notefont=\mdseries, notebraces={(}{)},
bodyfont=\normalfont,
postheadspace=0.5em,
spaceabove=6pt,
mdframed={
  skipabove=8pt,
  skipbelow=8pt,
  hidealllines=true,
  backgroundcolor={shadecolor},
  innerleftmargin=4pt,
  innerrightmargin=4pt}
]{shaded}

\usepackage{pgfplots}
\pgfplotsset{compat=1.17}
\usepackage{pgfplotstable}

\DeclareMathOperator{\argmininn}{argmin} 
 
\newcommand{\argmin}[1]{ \underset{#1}{\argmininn} \;}

\newcommand{\R}{\mathbb{R}}

\usepackage{caption}

\usepackage{graphicx}
\usepackage{amsmath}
\usepackage{wrapfig}
\usepackage{enumitem}
\usepackage{booktabs}
\usepackage{amsthm}
\newtheorem{proposition}{Proposition}

\newcommand{\dgs}[1]{{\scriptsize{#1}}}

\newcommand{\EMA}{\texttt{EMA}}
\usepackage{xcolor}
\definecolor{customgreen}{RGB}{61,157,115}
\definecolor{customred}{RGB}{213,112,95}
\definecolor{customgreen2}{RGB}{34,120,85}  
\definecolor{customred2}{RGB}{160,60,52}    
\definecolor{pastelgreen}{RGB}{158,206,185}
\definecolor{pastelred}{RGB}{234,183,175}
\usepackage{hyperref}       
\definecolor{cornflowerblue}{rgb}{0.39, 0.58, 0.93}
\definecolor{darkgray}{rgb}{0.0, 0.1, 0.5}
\definecolor{ocean}{rgb}{0.118, 0.392, 0.549}
\hypersetup{
    colorlinks=true, 
    linkcolor=ocean, 
    filecolor=ocean,      
    urlcolor=ocean,
    citecolor=ocean,
    }

\title{In Search of Adam’s Secret Sauce}

\author{%
  Antonio Orvieto~\thanks{\texttt{	antonio@tue.ellis.eu}.} \\
  ELLIS Institute T\"ubingen, MPI-IS\\
  T\"ubingen AI Center, Germany\\
  \And
  Robert M. Gower \\
  CCM, Flatiron Institute, Simons Foundation \\
  New York, US \\
}

\begin{document}

\maketitle

\vspace{-3mm}

\begin{abstract}
\vspace{-3mm}
Understanding the remarkable efficacy of Adam when training transformer-based language models has become a central research topic within the optimization community. To gain deeper insights, several simplifications of Adam have been proposed, such as the signed gradient and signed momentum methods. In this work, we conduct an extensive empirical study — training over 1,500 language models across different data configurations and scales — comparing Adam to several known simplified variants. We find that signed momentum methods are faster than SGD, but consistently underperform relative to Adam, even after careful tuning of momentum, clipping setting and learning rates. However, our analysis reveals a compelling option that preserves near-optimal performance while allowing for new insightful reformulations: constraining the Adam momentum parameters to be equal, $\beta_1=\beta_2$. Beyond robust performance, this choice affords new theoretical insights, highlights the {\color{customred2}``\textit{secret sauce}''} on top of signed momentum, and grants a precise statistical interpretation: we show that Adam in this setting implements a natural online algorithm for estimating the mean and variance of gradients—one that arises from a mean-field Gaussian variational inference perspective.

\end{abstract}

\section{Introduction}
\label{sec:intro}
\vspace{-3mm}

\begin{wrapfigure}[21]{R}{0.35\textwidth}
\vspace{-5mm}
\centering
  \includegraphics[width=0.33\textwidth]{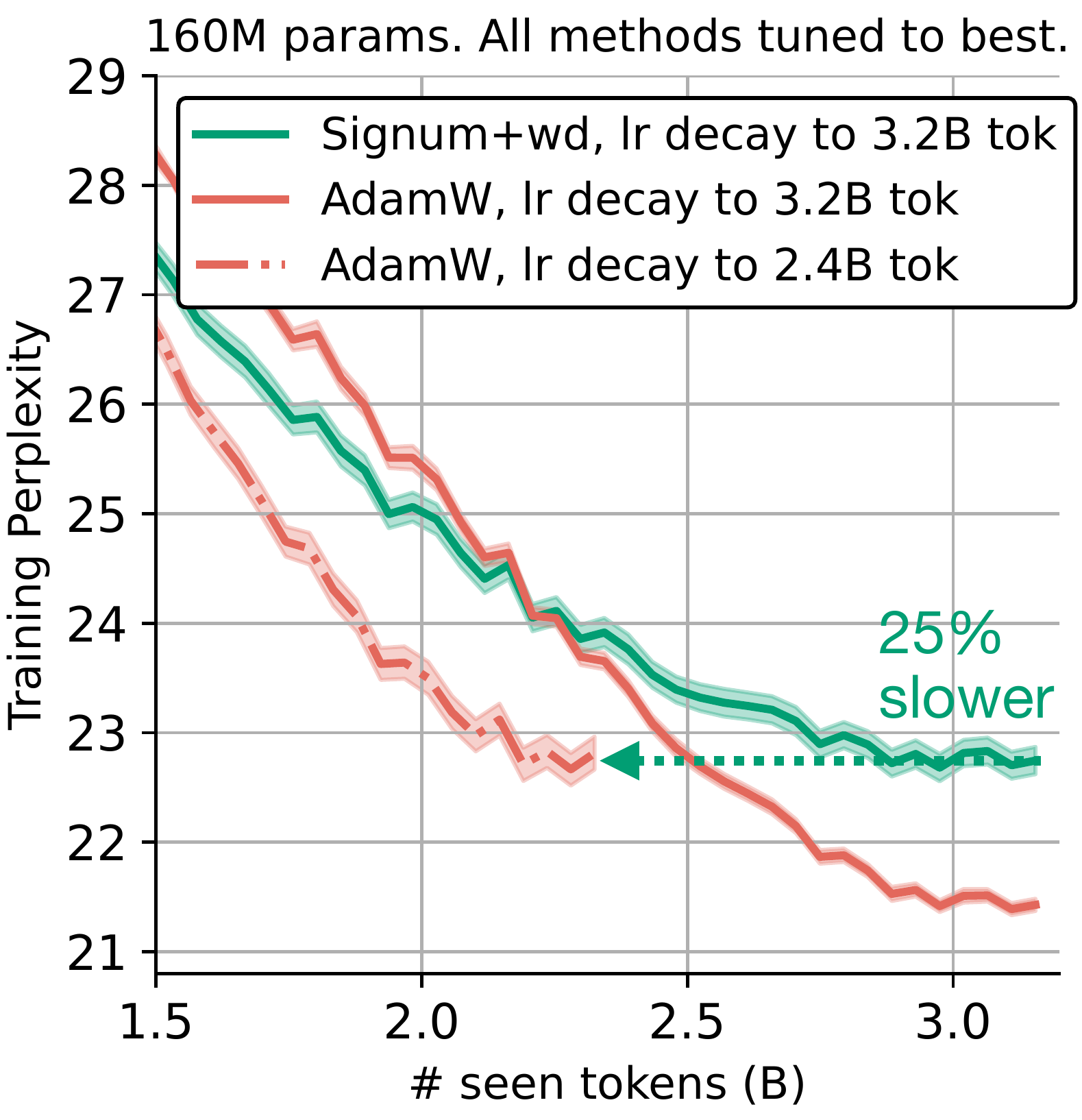}
  \vspace{-2mm}
  \caption{\small \textit{Pretraining on SlimPajama with Chinchilla-optimal~\citep{hoffmann2022training} scaling. Both momentum and learning rates for \Signum{} are extensively tuned~(\S\ref{sec:exp}).
  While \Signum{} closes $96\%$ of the perplexity gap between \Adam{} and \SGD{} with momentum~(Table~\ref{tab:algorithm-performance}), still results in a $25\%$ slowdown : \Adam{} achieves the same performance with 3/4 of the budget.}}
  \label{fig:cosine}
\end{wrapfigure}

Despite a decade of research into efficient and performant adaptive optimizers for deep learning, the \textit{de facto} choice for large-scale training today remains \Adam~\citep{kingma2014adam}, especially for training language models~(LMs)~\citep{grattafiori2024llama, liu2024deepseek}. At the root of this choice is the peculiar geometry of optimization landscapes induced by the transformer architecture~\citep{noci2022signal, zhang2024why}, as well as the noisy/unbalanced nature of tokenized text data~\citep{zhang2020why, kunstner2024heavy}.

In recent years, the surge of extremely large-scale and expensive-to-pretrain language models has further pushed the community to better understand \Adam's performance and to propose faster, efficient, and robust alternatives. Towards achieving this goal, contemporary studies~\citep{kunstner2023noise,bernstein2024old} have brought up a close similarity between the performance of \Adam{} and \SignSGD~\citep{bernstein2018signsgd} with momentum. 
While such results are extremely valuable to  forward our understanding, they are not precise enough : already at a scale of 160M parameters we found that extensive tuning of \Signum{}  (\SignSGD{} with momentum), while closing 96\% of the perplexity gap between \SGD{} and \Adam{}, results in a 25\% effective slowdown 
(Figure~\ref{fig:cosine}).

\begin{table}[t]
\centering
\vspace{-3mm}
\caption{\small \textit{~\textbf{(Signum closes 96\% of the perplexity gap between Adam and SGD)}
Validation perplexity comparison of widely used optimizers that interpolate between \SGD{} and \Adam{}, evaluated on a language modeling task (160M parameters, 3.2B SlimPajama tokens, sequence length 2048, batch size 256 -- Chinchilla optimal).
We report the mean and 2-sigma interval of validation perplexity (on 100M held-out tokens) across 3 initialization seeds. Weight decay is always decoupled
~\citep{loshchilov2018decoupled} and set to $0.1$~\citep{biderman2023pythia,liu2024deepseek} except for \SGD{} where we further tune~(\S\ref{app:more_experiments}). \RMSprop{} does not use momentum, and Gclip is global norm clipping to $1$~(before applying momentum), Cclip is coordinate-wise clipping~(after applying momentum). Other hyperparameters, for all other methods, are carefully tuned, see e.g. Figure~\ref{fig:big_sweep} and \S\ref{sec:exp}.\\ To optimally tune hyperparameters~(e.g. Figure~\ref{fig:big_sweep}), we performed a total of 582 full training runs.}}
\vspace{2mm}
{\small
\setlength{\tabcolsep}{4pt}
\begin{tabular}{lccccccc}
\hline
 & \Adam{}  &  \Signum{} & \RMSprop{}& \SGD +Cclip & \SignSGD{} & \SGD +Gclip & \SGD \\ \hline
Val ppl. & \textbf{21.86\dgs{$\pm$ 0.21}}  & \underline{23.23\dgs{$\pm$ 0.16}}& 27.04\dgs{$\pm$ 0.34}& 33.40\dgs{$\pm$ 0.39}  & 36.78\dgs{$\pm$ 0.57} & 37.76\dgs{$\pm$ 0.61} &  53.62\dgs{$\pm$ 5.14} \\ \hline
\end{tabular}}
\vspace{-3mm}
\label{tab:algorithm-performance}
\end{table}

While for large-scale training, the slowdown in Figure~\ref{fig:cosine} is not acceptable, it may seem unnecessary or anachronistic to further explain it, in light of recent algorithms claiming to have further improved the performance of \Adam, e.g. \Muon~\citep{jordan2024muon,liu2025muon, shah2025practical}, \Scion~\citep{pethick2025training}, and \Shampoo-based~\citep{gupta2018shampoo} methods such as \SOAP~\citep{vyas2025soap}. However, a close inspection of such optimizers reveals that, while gains over vanilla \Adam{} are solid, \textit{most of these methods still use \Adam{} on a specific subset of parameters}: For instance, in recent scaled-up versions of \Muon{}~\citep{liu2025muon, shah2025practical}, \Adam{} is used to update embedding, LM heads and normalization parameters~\footnote{Coincidentally, the ones that were shown to be most sensitive during training~\citep{zhao2025deconstructing, kunstner2024heavy}. \Scion{} claims a greater independence from \Adam, yet adopts an architecture where normalization layers have no trainable gain parameters. While results are promising, experiments in the usual setup are needed.}, and on the other parameters the \Muon{} update is normalized to have a similar RMS value similar to the \Adam{} update. Further, \SOAP's improvements stem from the application of \Adam{} in the preconditioner's eigenbasis.

The discussion above and the results in Figure~\ref{fig:cosine} inspires us to further dissect -- once again~\citep{balles2018dissecting} -- the mechanisms of \Adam{} compared to those of simpler methods in language modeling with transformers.\\
 Towards improving our understanding of \Adam{}, we make the following contributions:
\begin{itemize}[itemsep=0pt, leftmargin=*]
    \item We perform a large-scale evaluation~($\sim$ 10 thousand NVIDIA A100-SXM4-80GB GPU hours) of the performance of  established algorithms which claim a theoretical or empirical similarity/dissimilarity with \Adam{} on 160M parameters LMs with usual configurations~\citep{biderman2023pythia,black2022gpt}, at a compute-optimal budget on different datasets, at different batch-sizes and sequence lengths~(up to 2048 tokens). Crucially, we sweep over all momentum parameters for each method, for each learning rate in our grid -- for each of our settings. We find that, while clipping and sign descent methods can close most of the gap with \SGD, their performance is not satisfactory in comparison to \Adam{}~(Figure~\ref{fig:big_sweep}). We make all of our data, e.g. loss dynamics for all our settings, publicly available at~\url{https://github.com/aorvieto/SecretSauce}.
    \item 
    Through our extensive tuning of \Adam~(e.g., Figure~\ref{fig:big_sweep}, comprising 200 distinct hyperparameter settings), we identify one simplification that does perform well: that of setting $\beta_1 =\beta_2$~(emerging practical choice in contemporary literature~\citep{zhao2025deconstructing, shah2025practical, cattaneo2025tuning, zhang2025how}).
         We validate this finding~(\S\ref{sec:ablations}) at different batchsizes, data source, token budget, sequence length and larger scale~(410M): $\beta_1 =\beta_2$ performs at near-optimality across the majority of our experiments, see Figure~\ref{fig:correlation}. Given the breadth of our evaluation and the robustness of this finding, we recommend adopting 
$\beta_1 =\beta_2$ as the default setting for Adam for training language models at similar data and parameter scales. More broadly, this perspective suggests that Adam can be effectively treated as a one-parameter optimizer~(as \Signum{}).    
   \item  We show in \S\ref{sec:theory}, that reducing $\beta_1=\beta_2=\beta$ to a single parameter, leads to a surprising new interpretation of \Adam{}:  it is built on top of a nontrivial yet principled online method for estimating mean and variance of the gradients. Indeed, we can view the two momentum buffers as the result of an online Gaussian Variational inference method for tracking the mean and variance of the gradients as they change across iterations. This viewpoint directly adds to the discussion by~\citet{balles2018dissecting}, yet affords more precision induced by our empirically-informed simplification. 
    \item We offer a toy quadratic example illustrative of our findings, building on top of recent works on the peculiar landscape of transformer-based language modeling problems~\citep{noci2022signal,zhang2024why}. This example replicates the gaps between tuned \SGD{}, \Signum{}, and \Adam{} with equal betas in a 9-dimensional setting, helpful for future research and to gain intuition.
\end{itemize}



\begin{figure}
    \centering
\includegraphics[width=\linewidth]{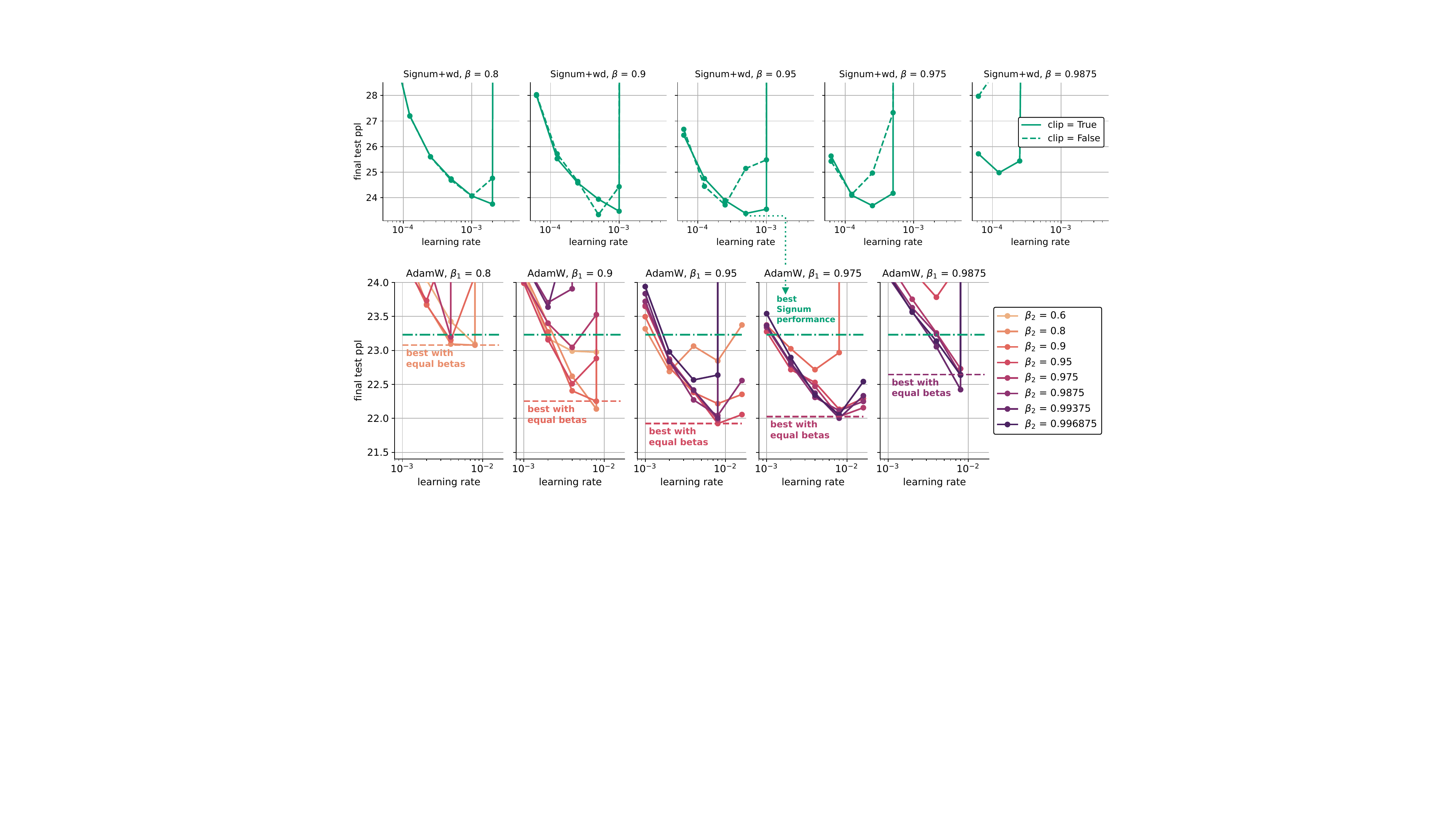}
    \caption{\small \textit{Training a total of \textbf{265 language models} with \textbf{160M} parameters with \textbf{3.2B} SlimPajama-627B tokens, sequence length of 2048, batch size of 256. Shown is the final test perplexity on 100M held-out tokens. Some underperforming runs are not shown to keep focus on the most interesting range. For a careful description of our tuning grid, see~\S\ref{app:exp-details}. \textbf{\color{customgreen}Takeaway 1:} Validation perplexity of highly tuned~(65 hyperparameter configurations) \Signum{} with weight decay 0.1 -- top row -- is around 23.23~(see Table~\ref{tab:algorithm-performance} for multiple seeds at optimal tuning). We ablate on the momentum parameter, learning rate, and presence of global clipping before averaging. The best performance of \Signum{} is reported as a green horizontal line on the second row~(200 \Adam{} runs, with weight decay of $0.1$). Most \Adam{} runs perform better than optimally tuned \Signum{}. \textbf{\color{customred}Takeaway 2:} For each $\beta_1$, the optimal corresponding $\beta_2$~(after learning rate tuning) is similar. The higher $\beta_1$, the higher $\beta_2$ for optimal performance~(optimal $\beta$s are correlated).}}
    \label{fig:big_sweep}
\end{figure}


 \section{Preliminaries and Related Works}
\label{sec:preliminaries}

For  a  signal $(s_k)_{k\in\mathbb{N}}$ and $\beta\in[0,1)$, we define the 
  $\beta$-normalized exponential moving average:
\begin{equation}
\label{eq:ema}
    \EMA_{\beta}[s_k] = \beta \EMA_{\beta}[s_{k-1}] + (1-\beta) s_k,\qquad \EMA_{\beta}[s_0] := s_0 \ \text{(or zero)}.
\end{equation}

The \Adam{} optimizer~\citep{kingma2014adam} without bias correction~\footnote{We show in Table~\ref{tab:ablation_ZI_BC} that the presence of bias correction does not affect our results at the best hyperparameter configuration. However, for all our runs, we use the Pytorch implementation including this factor, for simplicity.} takes the following form:
\begin{equation}
    w_{k+1} = w_k -\eta_k \left(\sqrt{\EMA_{\beta_2}[g_k^2]}+\epsilon\right)^{-1}\EMA_{\beta_1}[g_k]\qquad \tag{\Adam}
\end{equation}
where all division and multiplications are element-wise, $w_k, g_k\in \R^d$ are model parameters and gradients at iteration $k$, $\eta_k$ is the scheduled learning rate, and $\epsilon>0$ is a small constant. \RMSprop{}~\citep{tieleman2012lecture} is an earlier method that sets $\beta_1=0$.

One special case, and simplification, of \Adam{} is to consider  $\beta_1 = \beta_2 = \epsilon = 0$ which gives \SignSGD{}:
\begin{equation}
    w_{k+1} = w_k -\eta_k \texttt{sign}[g_k].
\tag{\SignSGD}
\end{equation}

A practical variant of \SignSGD, which has shown strong performance in language modeling~\citep{kunstner2023noise}, first computes an exponential moving average (EMA) -- or momentum -- of the gradients before applying the \texttt{sign} operator~\citep{bernstein2018signsgd}:
\begin{equation}
    w_{k+1} = w_k -\eta_k \texttt{sign}[\EMA_{\beta}[g_k]].
\tag{\Signum}
\end{equation}

In practice, every language modeling pipeline~(see e.g.~\citep{karpathy2022nanogpt}) incorporates some gradient clipping strategy~\citep{pascanu2013difficulty}, a component known to stabilize training in the autoregressive setting and to make gradients more robust to the stochasticity of language~\citep{zhang2020adaptive}. Global norm clipping~(that we abbreviate Gclip), processes gradients fresh out of the backward pass:
$$\texttt{Gclip}[g_k] = \min\left\{1, \frac{1}{\|g_k\|_2}\right\} g_k.$$
In our experiments, we start from vanilla SGD with momentum: $w_{k+1} = w_k -\eta_k \EMA_{\beta}[g_k]$ and ablate on the positive effect of Gclip before applying momentum. Regarding coordinate clipping~(Cclip), a softer version of \texttt{sign}, we consider applying it to $\EMA_{\beta}[g_k]$ -- in connection with \Signum{}.

\vspace{-3mm}
\paragraph{Research on Adam, a short summary.} Despite initial concerns on generalization~\citep{wilson2017marginal} and convergence~\citep{reddi2018convergence}, after the introduction of decoupled weight decay~(i.e., \AdamW~\citep{loshchilov2018decoupled}) \Adam{} rapidly became the de-facto standard optimizer in deep learning, with works highlighting its landscape adaptation properties~\citep{orvieto2022vanishing} and its debated connections to empirical Fisher preconditioning~\citep{kunstner2019limitations}.

 With the advent of Transformers~\citep{vaswani2017attention}, early works noticed an intriguing gap with \SGD{} performance in language modeling~\citep{xiong2020layer}~(much larger than what can be observed, e.g., in CNNs on image data), that was initially attributed to heavy-tail noise in text data~\citep{simsekli2019tail,zhang2020why} -- suggesting \Adam{} performance to be correlated with its adaptive coordinate clipping mechanism~\citep{zhang2020why}. 

  As models became larger and more hardware-demanding, interest spiked in the community to reduce the memory footprint of \Adam{}~\citep{li2023memory,zhang2024adam} and to search for more efficient options~\citep{chen2023symbolic,liu2023sophia}. Current trends,
draw an intriguing connection between \Adam{} and \SignSGD{}~\citep{bernstein2024old}, and in particular with its momentum variant: \Signum{}~\citep{bernstein2018signsgd}. 
This connection was first suggested in early attempts to understand \Adam{}’s empirical performance~\citep{balles2018dissecting},
and has recently gained renewed attention in light of transformer architectures and their heterogeneous optimization landscapes~\citep{noci2022signal, zhang2024why, tomihari2025understanding,kunstner2024heavy,zhao2025deconstructing}. 
These landscape-based arguments are now more compelling, as recent evidence shows that \Adam{} and signed momentum methods outperform SGD even in deterministic settings~\citep{kunstner2023noise}.


\vspace{-3mm}
\paragraph{Our approach.} 

Although recent literature highlights many connections between \Adam{} and simpler methods such as \Signum{}—which involve fewer hyperparameters, the computational demands of thoroughly studying \Adam{} on small- to medium-scale language models remain prohibitive for most academic optimization researchers. This challenge is amplified by the combinatorial explosion of hyperparameter configurations required for rigorous comparisons. In \S\ref{sec:exp}, we aim to provide a comprehensive empirical reference for optimizer performance across a range of language modeling settings. Our key findings are distilled into two main takeaways (Figure~\ref{fig:big_sweep}), which are further supported by theoretical insights in \S\ref{sec:theory}.

\section{Experiments}
\label{sec:exp}

In our experiments, we systematically explore Transformer-based language models using a nanoGPT~\citep{karpathy2022nanogpt} implementation\footnote{\url{https://github.com/Niccolo-Ajroldi/plainLM/tree/main}} enhanced by recent advancements such as Rotational Positional Embeddings~\citep{su2024roformer}, RMSNorm normalization~\citep{zhang2019root}, and SwiGLU activation functions~\citep{shazeer2020glu}. We adopt a robust training protocol inspired by successful practices established in large language models like LLaMa~\citep{touvron2023llama}, GPT-neox~\citep{black2022gpt}, GPT-J~\citep{wang2022gpt} and Pythia~\citep{biderman2023pythia}, leveraging techniques including bfloat16 precision, linear warm-up followed by a cosine annealing schedule~\citep{loshchilov2016sgdr}, and global gradient norm clipping~(unless specified). Our model configurations follow~\citep{biderman2023pythia} and are presented, alongside a detailed description of all tuning settings and resources, in \S\ref{app:exp-details}.

\subsection{Extensive benchmarking at 160M parameters}
\label{sec:160_standard}
We conduct 475 compute-optimal pretraining runs on the SlimPajama-627B dataset~\citep{cerebras2023slimpajama}, using a sequence length of 2048, a batch size of 256, and a decoupled weight decay of 0.1~\citep{loshchilov2018decoupled} (except for \SGD{}). 
 We always report validation perplexity on a held-out subset of 100M tokens.
 Results from these tuning sweeps are summarized in Table~\ref{tab:algorithm-performance}, Figure~\ref{fig:big_sweep}, and Appendix~\ref{app:tuning_other_methods}. The runs span the following configurations:
 
\noindent{\textbullet \; \SGD ~(131 runs):} Tuned parameters include weight decay (too large causes instability), global norm clipping~(Gclip). We also consider clipping coordinates after applying momentum~(Cclip). For all these options, momentum and learning rates are independently tuned.\\[0.2cm]
\noindent{\textbullet \; \RMSprop ~(48 runs):} Tuned parameters include momentum on the preconditioner and learning rate.\\[0.2cm]
\noindent{\textbullet \; \Signum~(70 runs):} Tuned parameters include global norm clipping, momentum, and learning rate.\\[0.2cm]
\noindent{\textbullet \; Momentum on \SignSGD~(35 runs):} This variant inverts the order of the \texttt{sign} and \EMA{} operations (and performs worse than \Signum). Clipping has no effect here due to the sign operation.\\[0.2cm]
\noindent{\textbullet \; \AdamW~(200 runs):} Tuned parameters include both momentum terms and the learning rate.\\[0.2cm]
Two additional seeds are provided for the best performing hyperparameter settings, see Table.~\ref{tab:algorithm-performance}. 

\textbf{Choice for betas grid.} 
While we vary the learning rate by powers of two, our choice of moving average parameters is guided by recent insights into \Adam{} scaling behavior
~\citep{malladi2022sdes,compagnoni2025adaptive}: we choose $\beta = 1-\kappa(1-\beta_{\text{base}})$ where $\beta_{\text{base}} = 0.9$ and $\kappa\in\{2^{-5},2^{-4}, \dots,2^2\}$. This makes it such that the accumulation factor $1/(1-\beta) = 1/(\kappa(1-\beta_{\text{base}}))$.


 \textbf{\color{customred}Takeaway 1}. 
 As shown in Figure~\ref{fig:big_sweep} and Table~\ref{tab:algorithm-performance}, optimally tuning \Signum{} with weight decay leads to significant improvements over standard \SGD{}, in line with recent findings~\citep{kunstner2023noise,zhao2025deconstructing}.
Nonetheless, \Adam{} consistently outperforms the alternatives across most settings, suggesting that it retains a key advantage—a "secret sauce"—that continues to set it apart from better-understood methods in large-scale optimization tasks.


This gap is not limited to this specific setup. In \S\ref{sec:ablations} we discuss results on another dataset~(Fineweb), with disabled weight decay, and shorter sequence lengths. Further, we ablate on other potential confounders~(initialization of moving averages, bias corrections, Adam $\epsilon$ value) in \S\ref{sec:sanity_checks}.   

\begin{wrapfigure}[28]{R}{0.30\textwidth}
\vspace{-5mm}
\centering
  \includegraphics[width=0.28\textwidth]{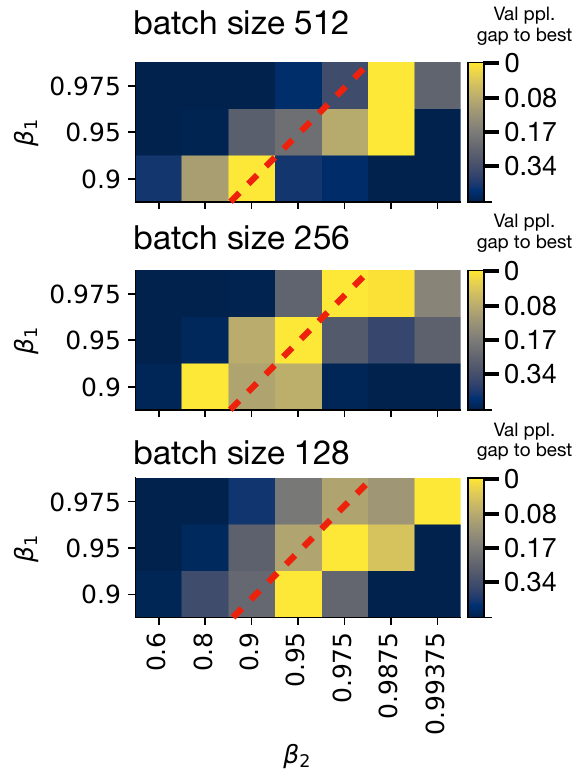}
  \vspace{-2mm}
  \caption{\small \textit{Summary of the results in \S\ref{app:more_batches}. At different batch sizes, for each  $\beta_1\in[0.9, 0.95, 0.975]$, we show the best-performing $\beta_2$~(highest score, yellow) and the gap between its performance and that of other options in the grid. We notice high correlation between beta values~(e.g., $\beta_2=0.9875$ is a terrible option at $\beta_1=0.9$, but a good one at $\beta_1=0.975$). While results are noisy, notice that $\beta_1=\beta_2$ never degrades performance more than $0.3$ points. In contrast~(Table~\ref{tab:algorithm-performance}, the gap with \Signum{} can be as high as $1.37$ points.}}
  \label{fig:correlation}
\end{wrapfigure}

\textbf{\color{customgreen}Takeaway 2 (a)}. In Figure~\ref{fig:big_sweep}, we clearly see that $\beta_1=\beta_2$ yields near-optimal performance in \Adam{}, for the five $\beta_1$ values we considered. In \S~\ref{sec:ablations} we show similar results at different batch sizes, different sequence lengths, and with disabled weight decay on a different dataset. We also extend this observation to 410M parameters models~(Figure~\ref{fig:big_sweep410}). This empirical finding serves as a basis for our theory in \S\ref{sec:theory}.

\textbf{\color{customgreen}Takeaway 2 (b)}. 
As a corollary to Takeaway~2, Figure~\ref{fig:correlation} shows that the best performance is not only achieved when $\beta_1 = \beta_2$, but also improves as the two values become closer. Among 500 runs on 160M-parameter models, we observe a clear correlation: lower loss is associated with smaller differences between $\beta_1$ and $\beta_2$. This suggests that gradient smoothing ($\beta_1$) and preconditioner smoothing ($\beta_2$) should not be treated as independent operations---optimal performance often arises when they act in concert. 

 To put to the test our second takeaway in \textbf{different training settings}, we consider shorter sequence lengths~(512, Figure~\ref{fig:adam_512}), higher/lower batch sizes~(Figure~\ref{fig:big_sweep_bs_big} \& Figure~\ref{fig:big_sweep_bs_small}), different data (Fineweb) and absence of weight decay~(Fig,~\ref{fig:160_fineweb}). See discussion in \S\ref{sec:ablations}.

\textbf{Standard choice for betas.} While in standard deep learning~(also Pytorch default) $\beta_2>\beta_1~(0.999, 0.9)$, in language modeling the choice $\beta_1=0.9, \beta_2=0.95$ is much more common. A lower value for $\beta_2$ was shown to help mitigate loss spikes~\citep{cattaneo2025tuning, compagnoni2025adaptive}, and several recent studies have started to adopt $\beta_1=\beta_2=0.95$ as a default~\citep{zhao2025deconstructing,shah2025practical, zhang2025how}. All our findings confirm this choice for tuning~(see e.g.~Figure~\ref{fig:big_sweep}), of which we evaluate validity extensively for several values of $\beta_1$. 

\textbf{Theoretical relations between betas.} We note that a correlation between $\beta$ parameters was also noted first by~\citet{reddi2018convergence, alacaoglu2020new}~for AMSgrad, and later by~\citet{zhang2022adam} for \Adam{}, where it is shown that if $\beta_2$ is large enough and $\beta_1<\sqrt{\beta_2}$, it converges to the neighborhood of critical points. Further,~\citet{xie2024implicit} showed that weight decay in \AdamW{} leads to convergence to a constrained minimizer only if $\beta_2>\beta_1$.

\subsection{Ablations}
\label{sec:ablations}

\paragraph{More Tokens.} We find our \textbf{\color{customred}Takeaway 2} to also hold at a higher token budget. In \S\ref{app:more_tokens}, we show a trend very similar to Fig.~\ref{fig:big_sweep} for models trained for $2\times$ the Chinchilla-optimal budget.

\vspace{-3mm}
\paragraph{Different batch size.} We find our \textbf{\color{customred}Takeaway 2} to be robust to batch size. In the same setting as Figure~\ref{fig:big_sweep} yet at a slightly lower compute budget due to hardware limitations~(2.5B parameters), we  find that, even at batch size 128 and 512 the choice $\beta_1=\beta_2$ yields near-optimal performance. This step involves training 500 models, see~\S\ref{app:more_batches} for visualizations similar to Figure~\ref{fig:big_sweep} and a discussion.


\vspace{-3mm}
\paragraph{Different sequence length.} In \S\ref{app:shorter_seq}, we find our \textbf{\color{customred}Takeaway 2} to also hold at shorter sequence length of 512~(Figure~\ref{fig:adam_512}). We note that here performance of \Signum{} is closer to that of \Adam{} compared to Figure~\ref{fig:big_sweep} -- yet, \Adam{} is still superior by a substantial margin~(~0.7 validation perplexity), \textbf{\color{customgreen}Takeaway 1}. This pattern agrees well with the results in~\citep{zhao2025deconstructing}, who found other methods to be competitive with \Adam{} at short context lengths. Our experiments in Figure~\ref{fig:adam_512} and Figure~\ref{fig:big_sweep} suggest that \Adam{} performance particularly shines at higher sequence lengths.

\vspace{-3mm}
\paragraph{Different data and weight decay.}In Figure~\ref{fig:160_fineweb} we test both \textbf{\color{customgreen}Takeaway 1} and \textbf{\color{customred}Takeaway 2} on Fineweb~\citep{penedo2024the}. We take this opportunity to also deactivate weight decay~($\lambda=0$), as the optimal \Signum{} learning rates in Figure~\ref{fig:big_sweep} suggest decoupled weight decay $w = w-\lambda\eta w$ acts differently for the two methods, likely needing different tuning. When deactivated, we still see a substantial gap between \Signum{} and \Adam{}, as well as strong performance with equal betas.

\begin{wrapfigure}[13]{R}{0.49\textwidth}
\vspace{-4mm}
\centering
  \includegraphics[width=0.48\textwidth]{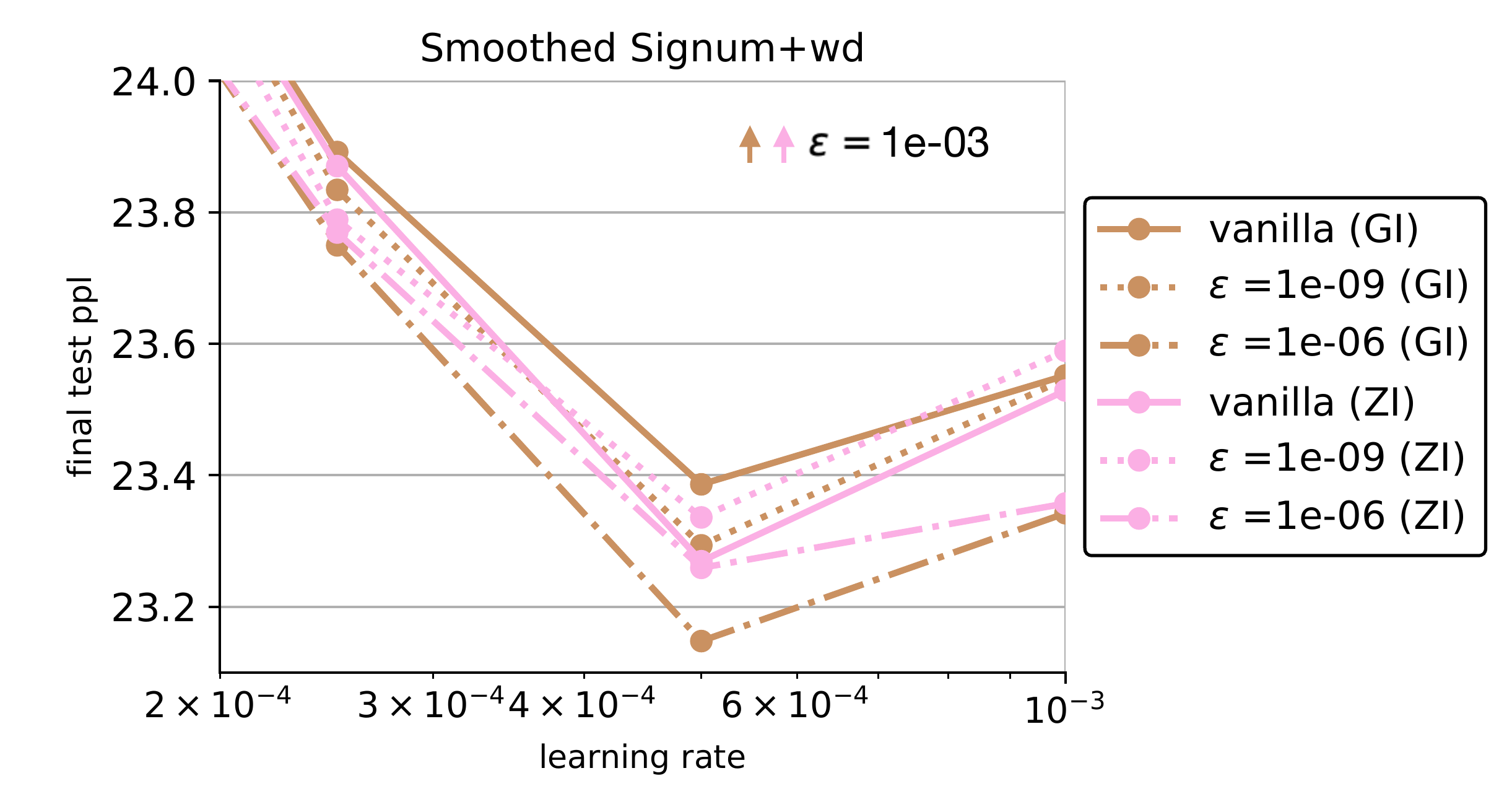}
  \vspace{-4mm}
  \caption{\small   \textit{Adding an $\epsilon$ mollifier to \Signum{}, i.e., using $m_k / (\sqrt{m_k^2} + \epsilon)$ offered little to no improvement. We also test both zero initialization (ZI) and gradient initialization (GI) for $m$, and find similar results with no significant improvement. $\epsilon=1e-3$ is significantly worse, hence is not shown. Similar finding:~Figure~\ref{fig:epsilon-signum-quad}.}}
  \label{fig:epsilon-signum}
\end{wrapfigure}

\vspace{-3mm}
\paragraph{Larger Models.}
We restrict our attention to the SlimPajama dataset and to validation of \textbf{\color{customred}Takeaway 2}. Results are presented in Figure~\ref{fig:big_sweep410}, comprising 44 full compute-optimal training runs of 410M parameter models, which confirm yet again strong and near-optimal performance at $\beta_1=\beta_2$.

\begin{figure}[t]
    \centering
    \vspace{-5mm}
\includegraphics[width=0.8\linewidth]{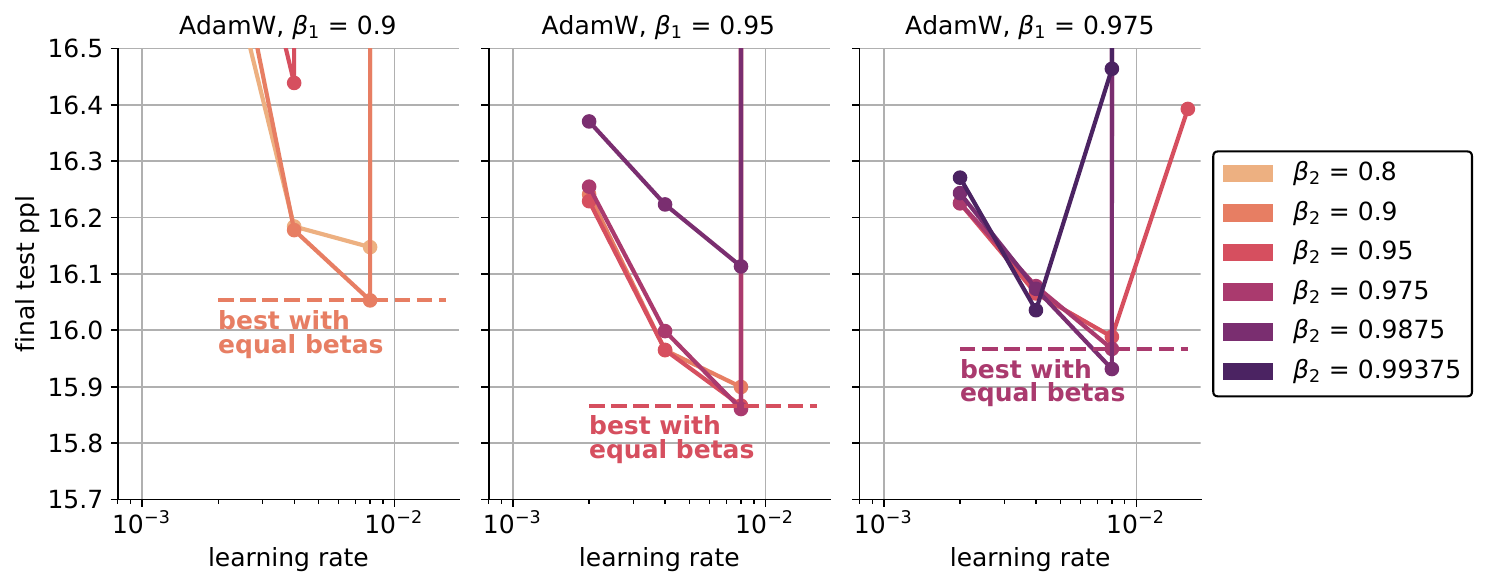}
    \caption{\small \textit{The final validation performance~(100M held-out tokens) for 44 trained LMs with 410M parameters trained on 8.2 B SlimPajama tokens~(Chinchilla-optimal). \textbf{Equal betas yields near-optimal performance}. We use gradient clipping and a batch size of 512~(scaled by 2 compared to Figure~\ref{fig:big_sweep}, as suggested by~\citet{zhang2025how}). Sequence length is 2048, weight decay is $0.1$. Note that the standard setting $(0.9, 0.95)$ is quite suboptimal here.}}
    \label{fig:big_sweep410}
    \vspace{-3mm}
\end{figure}

\subsection{Checking for confounders}
\label{sec:sanity_checks}

When comparing \Signum{} with \Adam{}, here are a few confounders that might affect results:  

   \vspace{-1mm}
\noindent{\bf The value of $\boldsymbol{\epsilon}$} in \Adam{} was shown to be important for numerical stability, and might affect performance~\citep{yuan2020eadam}. We show in Table~\ref{tab:ablation_eps} that one can choose an extremely small $\epsilon$ value in our setting. We cross-check the impact of including an $\epsilon$ factor in \Signum{}: we found that little can be gained from this strategy~(Figure~\ref{fig:epsilon-signum}). In short, we found that $\epsilon$ is not a crucial parameter in our setup. This is also liked to our findings on adaptive mollifiers, cf.~\S\ref{sec:theory}.


\vspace{-0mm}
\noindent{\bf Initialization of moving average parameters.} In Figure~\ref{fig:epsilon-signum} we also ablate on initialization of the moving average in \Signum{} and found no substantial differences. We perform this same ablation for \Adam{} and report comprehensive results with seeds in \S\ref{app:sanity_checks}.

\vspace{-3mm}
\paragraph{Bias correction.} While bias correction in \Adam{} is helpful in early training, final validation performance is almost unchanged, see the full training curve and results with seeds in \S\ref{app:sanity_checks}.

\begin{table}[ht]
\vspace{-3mm}
\centering
\caption{\small \textit{Effect of $\epsilon$ in \AdamW-- other parameters optimally tuned for $\epsilon = 10^{-8}$~(setting: Figure~\ref{fig:big_sweep}). 
All values between $\epsilon \in [10^{-6}, 10^{-15}]$ result in a similar validation perplexity.
}}
{\small
\begin{tabular}{lccccccc}
\hline
 & $\epsilon = 1e-3$  & $\epsilon = 1e-6$  & $\epsilon = 1e-8$  & $\epsilon = 1e-10$  & $\epsilon = 1e-12$  & $\epsilon = 1e-15$ \\ \hline
Val ppl & 23.34\dgs{$\pm$ 0.31} & 21.56\dgs{$\pm$ 0.19}  & 21.86\dgs{$\pm$ 0.21} & 21.87\dgs{$\pm$ 0.04} & 21.89\dgs{$\pm$ 0.2}& 21.91\dgs{$\pm$ 0.18} \\ \hline
\end{tabular}}
\label{tab:ablation_eps}
\end{table}

\section{New Viewpoints of \Adam}
\label{sec:theory}

We now show that restricting to the case $\beta_1 = \beta_2 = \beta$ yields a useful interpretation of \Adam{}. Since the \Adam{} update is coordinate-wise, it suffices to analyze a single scalar gradient $g_k \in \mathbb{R}$. Moreover, ablations (Table~\ref{tab:ablation_eps}, Table~\ref{tab:ablation_ZI_BC}) indicate that neither the $\epsilon$-term nor the bias correction significantly affect performance. Thus, for clarity, we set $\epsilon = 0$ and study the simplified \Adam{} update:
\begin{equation} \label{eq:adam-1d}
    d_k = 
    \frac{\EMA_{\beta}[g_k]}{\sqrt{\EMA_{\beta}[g_k^2]}}.
\end{equation}

We next rewrite~(proof in the Appendix) the update to explicitly highlight the role of variance.
\begin{restatable}{proposition}{propadammol}\label{prop:adammol}
Let $m_k = \EMA_{\beta}[g_k]$. Then the update~\eqref{eq:adam-1d} admits the equivalent representation:
\begin{equation}\label{eq:adam-denom-msq-var}
    d_k \;=\; {\color{customgreen2} \frac{m_k}{\sqrt{m_k^2 + {\color{customred2}\beta \, \EMA_\beta[(m_{k-1} - g_k)^2]}}}}.
\end{equation}
\end{restatable}

This shows that the denominator depends on the exponential moving average of the squared deviation between the momentum $m_{k-1}$ and the incoming gradients $g_k$, with an \textbf{interesting multiplier} $\beta$. As we demonstrate in the next section, this quantity is in fact an online estimator of the gradient variance.

\subsection{\Adam{}  Estimates Mean and Variance using  Variational Inference }
\label{sec:estimation}

We show that \Adam{} admits a natural interpretation as an online variational inference method, where
\[
m_k := \EMA_{\beta}[g_k] \quad \text{and} \quad \sigma_k^2 := \beta\, \EMA_{\beta}[(m_{k-1}-g_k)^2]
\]
correspond to online estimates of the mean and variance of the stochastic gradients. We reintroduce \Adam{} through this lens.

Suppose we are given a sequence of stochastic gradients $\{g_1, \ldots, g_k\}$, where each $g_k$ is sampled from an unknown Gaussian distribution whose mean and variance may vary with $k$. Rather than taking steps directly along these noisy gradients, we aim to estimate their mean and variance online and use these estimates to define a more informed search direction.

At iteration $k$, let $(m_k, \var_k)$ denote our current estimates of the gradient mean and variance, respectively. Upon receiving a new gradient sample $g_{k+1} \sim \mathcal{N}(m, \var)$ with unknown $(m, \var)$, we wish to update our estimates to $(m_{k+1}, \var_{k+1})$ so that it becomes more \emph{likely} that $g_{k+1}$ was drawn from $\mathcal{N}(m_{k+1}, \var_{k+1})$. Since we also expect the underlying distribution to vary slowly over time, we prefer that $\mathcal{N}(m_{k+1}, \var_{k+1})$ remain close to the previous estimate $\mathcal{N}(m_k, \var_k)$. These two goals—fitting the new observation and ensuring smooth updates—can be traded off via a regularized maximum likelihood problem:
\begin{equation} \label{eq:log-like}
    \min_{m, \var \geq 0} \; - \log p(g_{k+1} \mid m, \var) + \, \tfrac{1}{\lambda }\mathrm{KL}\left(\mathcal{N}(m_k, \var_k)\,\|\,\mathcal{N}(m, \var)\right),
\end{equation}
where $p(g_{k+1} \mid m, \var)$ is the Gaussian likelihood, $\lambda \geq 0$ is a regularization parameter, and $\mathrm{KL}$ denotes the Kullback--Leibler divergence:
\begin{align}
    -\log p(g_{k+1} \mid m, \var) &= \frac{1}{2} \log \var + \frac{1}{2\var}(g_{k+1} - m)^2, \\
    \mathrm{KL}\left(\mathcal{N}(m_k, \var_k)\,\|\,\mathcal{N}(m, \var)\right) &=
    \frac{1}{2} \left[ \frac{\var_k}{\var} + \frac{(m_k - m)^2}{\var} - 1 - \log\left( \frac{\var_k}{\var} \right) \right]. \label{eq:logp-KL}
\end{align}

The following result, proved in the appendix, characterizes the solution of \eqref{eq:log-like}, showing that the moving averages used in \Adam{} correspond exactly to an instance of online variational inference:

\begin{restatable}{theorem}{theoVIAdam} \label{thm:theoVIAdam}
Let $\beta = \frac{1}{1 + \lambda}$. Then the solution to the optimization problem~\eqref{eq:log-like} is given by
\begin{align}
    m_{k+1} &= \beta m_k + (1 - \beta) g_{k+1} = \EMA_\beta[g_{k+1}], \label{eq:mom-VI-view} \\
    \var_{k+1} &= \beta \var_k + \beta(1 - \beta)(m_k - g_{k+1})^2 = \beta\, \EMA_\beta\left[(m_k - g_{k+1})^2\right]. \label{eq:varVI-view}
\end{align}
\end{restatable}

As a consequence, the \Adam{} update direction in~\eqref{eq:adam-denom-msq-var} can be rewritten as
\begin{equation} \label{eq:Adam-var-view}
    d_k \;=\; \frac{\color{customgreen2}  m_k}{\sqrt{{\color{customgreen2} m_k^2} +{\color{customred2} \beta \EMA{_\beta}[(m_{k-1} - g_k)^2]}}} \; =\; \frac{\color{customgreen2}  m_k}{\sqrt{{\color{customgreen2} m_k^2} + {\color{customred2} \var_k}}} = \frac{\operatorname{sign}({\color{customgreen2} m_k})}{\sqrt{1 + {\color{customred2} \sigma_k^2} /{\color{customgreen2}  m_k^2}}}.
\end{equation}
This shows that \Adam{} may be interpreted as an \emph{adaptive mollified} variant of \Signum{}, where the mollification depends on the local noise-to-signal ratio. This mollified viewpoint
 aligns well with one of the first papers on understanding \Adam{}~\citep{balles2018dissecting}, as discussed after Proposition~\ref{prop:adammol}. 
 
Using these insights, we can better formalize the \emph{noise-to-signal} interpretation of \Adam{}~\citep{balles2018dissecting}~(see also~\S\ref{sec:balles}). Let $m_k/\sigma_k$ denote the signal-to-noise ratio (SNR). We show that \Adam{} can be viewed as a steepest descent method whose trust region is modulated by the SNR.

To build this connection, consider first the \Signum{} update. It corresponds to the steepest descent direction under an $\ell_\infty$-norm constraint~\citep{balles2018dissecting}, solving
\begin{equation}\label{eq:abs-descent}
-\mbox{sign}(m_k) \; = \; \argmin{\theta \in \R} -m_k \cdot \theta \quad \text{subject to } |\theta | \leq 1.
\end{equation}
That is, \Signum{} selects the direction most aligned with $-m_k$ within a unit trust region.

In contrast, \Adam{} can be interpreted as a steepest descent method with a variable trust region, defined by the (inverse) signal-to-noise ratio:
\begin{equation}\label{eq:adam-descent}
-\frac{\mbox{sign}(m_k)}{\sqrt{1 + \sigma_k^2 / m_k^2}} \; = \; \argmin{\theta \in \R} -m_k \cdot \theta \quad \text{subject to } |\theta | \leq \frac{1}{\sqrt{1 + \sigma_k^2 / m_k^2}}.
\end{equation}
Here, the effective step size shrinks when the noise dominates the signal, and expands toward $1$ as uncertainty decreases. In this sense, \Adam{} adapts its update magnitude according to a confidence-weighted trust region.

\subsection{Comparison with Balles and Hennig [2018]} 
\label{sec:balles}
 \citet{balles2018dissecting} first drew a connection between \Adam{}, \Signum{} and Signal-to-noise Ratio regularization.
 Their observation was as follows. Let $m_k = \EMA_{\beta_1}[g_k]$, and  $v_k = \EMA_{\beta_2}[g_k^2]$.
We can trivially re-write the \Adam{} direction as
$$d_k = \frac{m_k}{\sqrt{v_k}} = {\color{customgreen2}\frac{m_k}{\sqrt{m_k^2 + {\color{customred2}v_k- m_k^2}}}}.$$
If we now \emph{assume} that  $\sigma_k^2 := v_k- m_k^2$ is a measure of variance, then dividing the \Adam{} direction through by $|{\color{customgreen2} m_k}|$, as done in~\eqref{eq:Adam-var-view}, we arrive at a Signal-to-noise Ratio regularized variant of the \Signum{} method. In particular, as the noise goes to zero ($\sigma_k^2 \rightarrow 0$), we arrive at the \Signum{} method.

The missing piece in their insight was to show when and if the term $v_k- m_k^2$ is a measure of variance.

We show that $\beta_1=\beta_2$, a choice that was not common\footnote{Default parameters have for long been $\beta_1=0.9$, $\beta_2=0.999$, see \url{ https://docs.pytorch.org/docs/stable/generated/torch.optim.Adam.html}.} at the time of~\citet{balles2018dissecting}, allows for more precise claims: Proposition~\ref{prop:adammol} shows that when $\beta_1=\beta_2=\beta$ the term $v_k- m_k^2$ is precisely equal to $\beta \EMA_\beta[(m_{k-1} - g_k)^2]$, which in turn is a online estimate of variance (Theorem~\ref{thm:theoVIAdam}).
 We further show that $v_k- m_k^2$ only has a precise variance interpretation for the case $\beta_1=\beta_2$: indeed, we prove in \S\ref{sec:specific_betas} that \Adam{} can be represented as
\begin{equation}
\label{eq:adam_rewrite}
d_k = \frac{m_k}{\sqrt{m_k^2 + \gamma \, \EMA_\tau[(a m_{k-1} - b g_k)^2]}}
\end{equation}
for some $a, b, \gamma \in \mathbb{R}$ and $\tau \in (0,1)$ \emph{if and only if} $\beta_1 = \beta_2$. In other words, connecting $v_k- m_k^2$ to variance estimation, and in turn \Adam{} to an SNR-controlled trust region method~\eqref{eq:adam-descent}, can only be done precisely for the case of equal betas.

\paragraph{Ablating hyperparameters in our reformulation.} While~\eqref{eq:adam_rewrite} reduces to Adam with equal betas if and only if $a,b=1$ and $\beta=\gamma=\tau$, we found it interesting to consider~\eqref{eq:adam_rewrite}, with $a=b=1$, as a new method with no precise connection to simultaneous variance and mean estimation, with hyperparameters $\beta,\gamma,\tau$. In \S\ref{sec:adavar}, we train 150 additional language models ablating on such parameters, and found no advantage in setting $\beta\ne \tau$ or $\tau\ne\gamma$. We believe such evidence further strengthens our claims: best performance is aligned to the theoretical choice $\tau=\gamma=\beta$.

\section{Why an adaptive trust region? Insights from heterogeneous quadratics}
\label{sec:quadratic}

\begin{figure}[ht]
  \vspace{-1mm}
\centering
  \includegraphics[width=\linewidth]{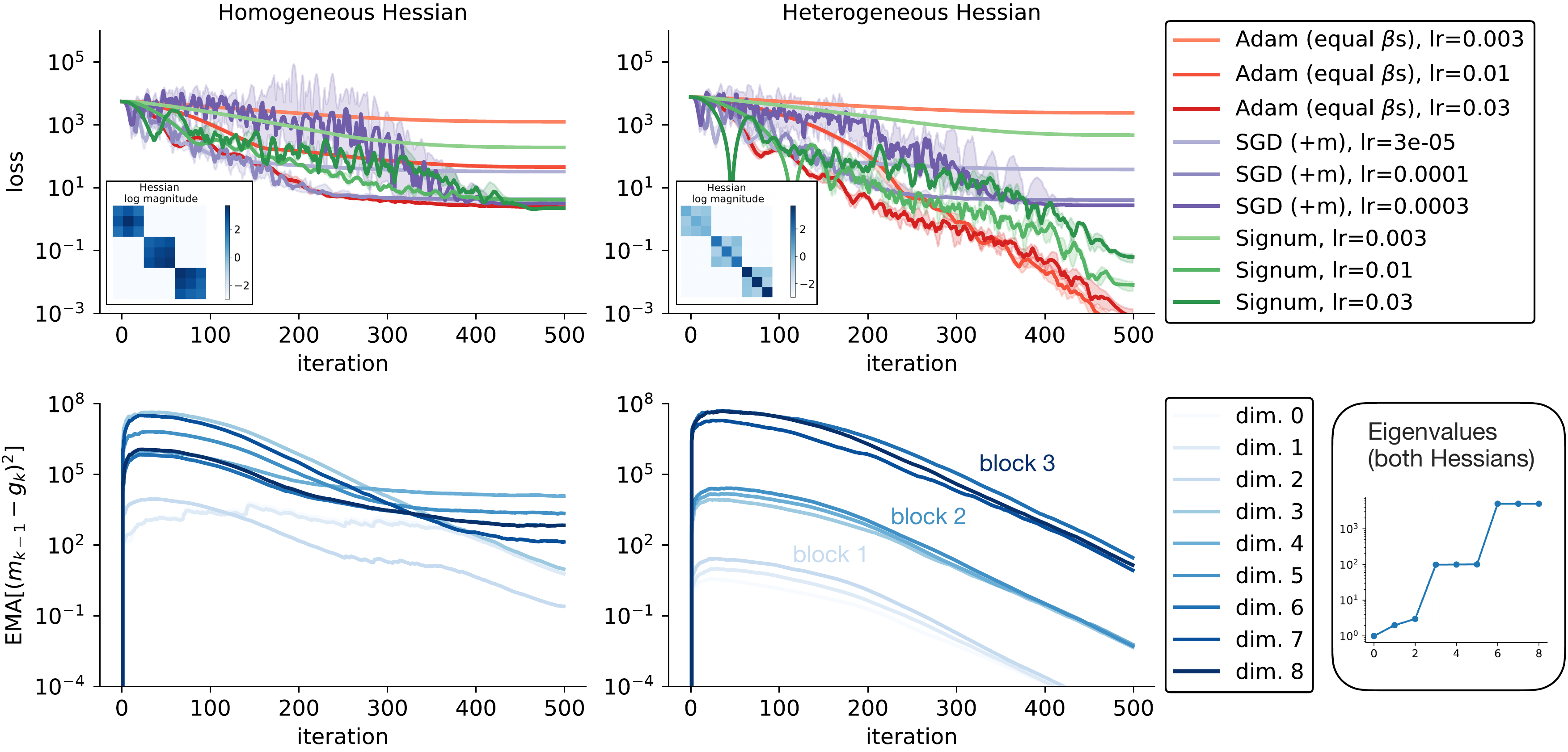}
  \caption{\small \textit{
\textbf{Top row:} Training performance (median and 25\%/75\% quantiles over 10 seeds) of \SGD{}, \Signum{}, and \Adam{} on two 9-dimensional convex quadratic problems (\S\ref{app:quadratic_details}) inspired by~\citet{zhang2024why}. All optimizers use moving average parameters set to $0.95$, with a 10\% warmup followed by cosine decay to zero. Both problems share the same Hessian eigenspectrum and have a $3 \times 3$ block structure. The landscape on the \emph{left} is \emph{homogeneous}, with each block containing both large and small eigenvalues. The landscape on the \emph{right} is \emph{heterogeneous}, with each block having eigenvalues of different magnitudes. In this setting, \Adam{} clearly outperforms \SGD{}, with \Signum{} closing part of the gap.
\textbf{Bottom row:} Dynamics of the variance term in Proposition~\ref{prop:adammol}. The value of this term varies both across iterations and across blocks, adapting to the local curvature structure. This adaptive behavior improves performance over \Signum{} in the heterogeneous setting.
  }}
  \label{fig:quadratic}
  \vspace{-3mm}
\end{figure}

While our theoretical analysis in \S\ref{sec:theory} offers a new perspective on \Adam{}, it is not tied to any specific architecture. To enhance intuition and provide a controlled setting for future work, we validate our findings on a simplified model of transformer loss landscapes introduced by~\citet{zhang2024why}, building on signal propagation theory~\citep{noci2022signal}.\\
As noted in~\citet{zhang2024why, kunstner2024heavy, zhao2025deconstructing}, the landscape of autoregressive language models is highly heterogeneous: Hessian blocks associated with semantically distinct parameter groups (e.g., normalization layers, embeddings, or softmax-related parameters) exhibit markedly different eigenspectra and thus demand different learning rates. In contrast to homogeneous models (e.g., CNNs), this heterogeneity is where \Adam{} significantly outperforms \SGD{}~\citep[cf.][]{zucchet2024recurrent}.

Figure~\ref{fig:quadratic} illustrates this point. On a toy heterogeneous quadratic landscape, tuned \Adam{} with equal $\beta$ values substantially outperforms tuned \SGD{} with momentum, echoing results from~\citet{zhang2024why}. We also observe that \Signum{} closes much of the gap but still falls short of \Adam{}. This is consistent with our findings in Table~\ref{tab:algorithm-performance} for language models.

\begin{wrapfigure}[14]{R}{0.40\textwidth}
\vspace{-6mm}
\centering
  \includegraphics[width=0.39\textwidth]{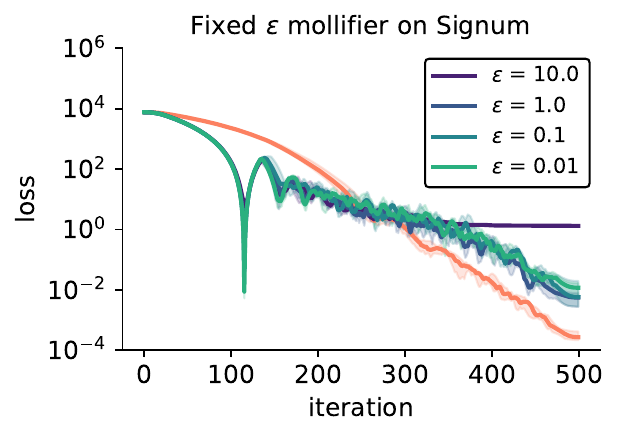}
  \vspace{-4mm}
  \caption{\small  \textit{Counterpart of Figure~\ref{fig:epsilon-signum} for the heterogeneous quadratic example. We do not observe gains with a fixed mollifier $m_k/\sqrt{m_k^2+\epsilon}$. Placing inside or outside $\sqrt{\cdot}$ has no qualitative effect after tuning.}}
  \label{fig:epsilon-signum-quad}
\end{wrapfigure}
In Proposition~\ref{prop:adammol}, we showed that the key difference between \Signum{} and \Adam{} lies in the variance correction term $\beta \texttt{EMA}_\beta[(m_{k-1} - g_k)^2]$ in the denominator. Understanding how this term evolves is essential: it cannot be approximated by a constant. In the second row of Figure~\ref{fig:quadratic}, we observe that the variance estimate not only varies over time, but also differs in scale across the three blocks—mimicking the parameter groupings in transformer models. This block-wise variation reinforces the idea that the variance term dynamically adapts to the local curvature and cannot be substituted by a fixed value. In Figure~\ref{fig:epsilon-signum-quad} and \ref{fig:epsilon-signum}, we show a similar effect in heterogeneous quadratic and language models, respectively: replacing $\beta \texttt{EMA}_\beta[(m_{k-1} - g_k)^2]$ with a fixed constant $\epsilon$ cannot provide the same adaptive effect.

\section{Conclusion}
\label{sec:conclusion}


We have presented an extensive numerical study of \Adam{}, comparing it against several proposed simplifications. Our main finding is that, on generative language modeling tasks, \Adam{} significantly outperforms these simplified variants. Notably, we observe that setting $\beta_1 = \beta_2$ is often optimal or near-optimal. Based on this observation, we recommend \Adam{} with $\beta_1 = \beta_2$ as a simplified model, and we provide a new variational inference interpretation for this setting.

Our findings come with some limitations. First, our numerical experiments fix a grid over the hyperparameters; the results are therefore sensitive to the choice of grid, and different grids may lead to different conclusions. However, for all our hyperparameters, we show explicitly all tuning curves demonstrating that we are always at optimality inside the grid~(and not at the edge). Second, while $\beta_1 = \beta_2$ often performs well, we note that at small batch sizes, Figure~\ref{fig:correlation} suggests a slight shift. Finally, although Theorem~\ref{thm:theoVIAdam} shows that \Adam{}’s two momentum buffers can be interpreted as online estimates of the gradient’s mean and variance, it does not explain why these estimates should be arranged into the quotient used in \Adam{}~\eqref{eq:Adam-var-view}. Lemma~1 in~\citep{balles2018dissecting} can provide a starting point to further dissect this interesting choice and explore alternatives.

\section*{Acknowledgements}
We would like to thank Niccolo Ajroldi, Sam Liang, Weronika Ormaniec, and Enea Monzio Compagnoni for their comments. We additionally thank the NeurIPS 2025 and ICML 2025 HiLD workshop reviewers for their valuable feedback and references. Antonio Orvieto acknowledges the financial support of the Hector Foundation, and is thankful for the compute made available by MPI-IS and the Tübingen AI ecosystem.

\newpage

\bibliography{main}

\begin{thebibliography}{59}
\providecommand{\natexlab}[1]{#1}
\providecommand{\url}[1]{\texttt{#1}}
\expandafter\ifx\csname urlstyle\endcsname\relax
  \providecommand{\doi}[1]{doi: #1}\else
  \providecommand{\doi}{doi: \begingroup \urlstyle{rm}\Url}\fi

\bibitem[Alacaoglu et~al.(2020)Alacaoglu, Malitsky, Mertikopoulos, and Cevher]{alacaoglu2020new}
Ahmet Alacaoglu, Yura Malitsky, Panayotis Mertikopoulos, and Volkan Cevher.
\newblock A new regret analysis for adam-type algorithms.
\newblock In \emph{International conference on machine learning}, pages 202--210. PMLR, 2020.

\bibitem[Balles and Hennig(2018)]{balles2018dissecting}
Lukas Balles and Philipp Hennig.
\newblock Dissecting {Adam}: The {S}ign, {M}agnitude and {V}ariance of {S}tochastic {G}radients.
\newblock In \emph{ICML}, 2018.

\bibitem[Bernstein and Newhouse(2024)]{bernstein2024old}
Jeremy Bernstein and Laker Newhouse.
\newblock Old optimizer, new norm: An anthology.
\newblock \emph{arXiv preprint arXiv:2409.20325}, 2024.

\bibitem[Bernstein et~al.(2018)Bernstein, Wang, Azizzadenesheli, and Anandkumar]{bernstein2018signsgd}
Jeremy Bernstein, Yu-Xiang Wang, Kamyar Azizzadenesheli, and Animashree Anandkumar.
\newblock signsgd: Compressed optimisation for non-convex problems.
\newblock In \emph{ICML}, 2018.

\bibitem[Biderman et~al.(2023)Biderman, Schoelkopf, Anthony, Bradley, O’Brien, Hallahan, Khan, Purohit, Prashanth, Raff, et~al.]{biderman2023pythia}
Stella Biderman, Hailey Schoelkopf, Quentin~Gregory Anthony, Herbie Bradley, Kyle O’Brien, Eric Hallahan, Mohammad~Aflah Khan, Shivanshu Purohit, USVSN~Sai Prashanth, Edward Raff, et~al.
\newblock Pythia: A suite for analyzing large language models across training and scaling.
\newblock In \emph{ICML}, 2023.

\bibitem[Black et~al.(2022)Black, Biderman, Hallahan, Anthony, Gao, Golding, He, Leahy, McDonell, Phang, et~al.]{black2022gpt}
Sid Black, Stella Biderman, Eric Hallahan, Quentin Anthony, Leo Gao, Laurence Golding, Horace He, Connor Leahy, Kyle McDonell, Jason Phang, et~al.
\newblock Gpt-neox-20b: An open-source autoregressive language model.
\newblock \emph{arXiv preprint arXiv:2204.06745}, 2022.

\bibitem[Cattaneo and Shigida(2025)]{cattaneo2025tuning}
Matias~D. Cattaneo and Boris Shigida.
\newblock Tuning adam(w): Default $\beta_2$ may be too large, 2025.
\newblock URL \url{https://mdcattaneo.github.io/papers/Cattaneo-Shigida_2025_TuningAdam.pdf}.

\bibitem[Chen et~al.(2023)Chen, Liang, Huang, Real, Wang, Pham, Dong, Luong, Hsieh, Lu, et~al.]{chen2023symbolic}
Xiangning Chen, Chen Liang, Da~Huang, Esteban Real, Kaiyuan Wang, Hieu Pham, Xuanyi Dong, Thang Luong, Cho-Jui Hsieh, Yifeng Lu, et~al.
\newblock Symbolic discovery of optimization algorithms.
\newblock \emph{Advances in neural information processing systems}, 36:\penalty0 49205--49233, 2023.

\bibitem[Chowdhery et~al.(2023)Chowdhery, Narang, Devlin, Bosma, Mishra, Roberts, Barham, Chung, Sutton, Gehrmann, et~al.]{chowdhery2023palm}
Aakanksha Chowdhery, Sharan Narang, Jacob Devlin, Maarten Bosma, Gaurav Mishra, Adam Roberts, Paul Barham, Hyung~Won Chung, Charles Sutton, Sebastian Gehrmann, et~al.
\newblock Palm: Scaling language modeling with pathways.
\newblock \emph{Journal of Machine Learning Research}, 24\penalty0 (240):\penalty0 1--113, 2023.

\bibitem[Compagnoni et~al.(2025)Compagnoni, Liu, Islamov, Proske, Orvieto, and Lucchi]{compagnoni2025adaptive}
Enea~Monzio Compagnoni, Tianlin Liu, Rustem Islamov, Frank~Norbert Proske, Antonio Orvieto, and Aurelien Lucchi.
\newblock Adaptive methods through the lens of {SDE}s: Theoretical insights on the role of noise.
\newblock In \emph{ICLR}, 2025.

\bibitem[Dao et~al.(2022)Dao, Fu, Ermon, Rudra, and R{\'e}]{dao2022flashattention}
Tri Dao, Dan Fu, Stefano Ermon, Atri Rudra, and Christopher R{\'e}.
\newblock Flashattention: Fast and memory-efficient exact attention with io-awareness.
\newblock \emph{Advances in neural information processing systems}, 35, 2022.

\bibitem[Grattafiori et~al.(2024)Grattafiori, Dubey, Jauhri, Pandey, Kadian, Al-Dahle, Letman, Mathur, Schelten, Vaughan, et~al.]{grattafiori2024llama}
Aaron Grattafiori, Abhimanyu Dubey, Abhinav Jauhri, Abhinav Pandey, Abhishek Kadian, Ahmad Al-Dahle, Aiesha Letman, Akhil Mathur, Alan Schelten, Alex Vaughan, et~al.
\newblock The llama 3 herd of models.
\newblock \emph{arXiv preprint arXiv:2407.21783}, 2024.

\bibitem[Gupta et~al.(2018)Gupta, Koren, and Singer]{gupta2018shampoo}
Vineet Gupta, Tomer Koren, and Yoram Singer.
\newblock Shampoo: Preconditioned stochastic tensor optimization, 2018.

\bibitem[Hoffmann et~al.(2022)Hoffmann, Borgeaud, Mensch, Buchatskaya, Cai, Rutherford, Casas, Hendricks, Welbl, Clark, et~al.]{hoffmann2022training}
Jordan Hoffmann, Sebastian Borgeaud, Arthur Mensch, Elena Buchatskaya, Trevor Cai, Eliza Rutherford, Diego de~Las Casas, Lisa~Anne Hendricks, Johannes Welbl, Aidan Clark, et~al.
\newblock Training compute-optimal large language models.
\newblock \emph{arXiv preprint arXiv:2203.15556}, 2022.

\bibitem[Jordan et~al.(2024)Jordan, Jin, Boza, Jiacheng, Cesista, Newhouse, and Bernstein]{jordan2024muon}
Keller Jordan, Yuchen Jin, Vlado Boza, You Jiacheng, Franz Cesista, Laker Newhouse, and Jeremy Bernstein.
\newblock Muon: An optimizer for hidden layers in neural networks, 2024.
\newblock URL \url{https://kellerjordan.github.io/posts/muon/}.

\bibitem[Karpathy(2022)]{karpathy2022nanogpt}
Andrej Karpathy.
\newblock Nanogpt, 2022.

\bibitem[Kingma and Ba(2014)]{kingma2014adam}
Diederik~P Kingma and Jimmy Ba.
\newblock Adam: A method for stochastic optimization.
\newblock \emph{arXiv preprint arXiv:1412.6980}, 2014.

\bibitem[Kunstner et~al.(2019)Kunstner, Hennig, and Balles]{kunstner2019limitations}
Frederik Kunstner, Philipp Hennig, and Lukas Balles.
\newblock Limitations of the empirical fisher approximation for natural gradient descent.
\newblock In \emph{Advances in Neural Information Processing Systems}, 2019.

\bibitem[Kunstner et~al.(2023)Kunstner, Chen, Lavington, and Schmidt]{kunstner2023noise}
Frederik Kunstner, Jacques Chen, Jonathan~Wilder Lavington, and Mark Schmidt.
\newblock Noise is not the main factor behind the gap between sgd and adam on transformers, but sign descent might be.
\newblock In \emph{ICLR}, 2023.

\bibitem[Kunstner et~al.(2024)Kunstner, Milligan, Yadav, Schmidt, and Bietti]{kunstner2024heavy}
Frederik Kunstner, Alan Milligan, Robin Yadav, Mark Schmidt, and Alberto Bietti.
\newblock Heavy-tailed class imbalance and why adam outperforms gradient descent on language models.
\newblock \emph{Advances in Neural Information Processing Systems}, 2024.

\bibitem[Li et~al.(2023)Li, Chen, and Zhu]{li2023memory}
Bingrui Li, Jianfei Chen, and Jun Zhu.
\newblock Memory efficient optimizers with 4-bit states.
\newblock \emph{Advances in Neural Information Processing Systems}, 36:\penalty0 15136--15171, 2023.

\bibitem[Liu et~al.(2024)Liu, Feng, Xue, Wang, Wu, Lu, Zhao, Deng, Zhang, Ruan, et~al.]{liu2024deepseek}
Aixin Liu, Bei Feng, Bing Xue, Bingxuan Wang, Bochao Wu, Chengda Lu, Chenggang Zhao, Chengqi Deng, Chenyu Zhang, Chong Ruan, et~al.
\newblock Deepseek-v3 technical report.
\newblock \emph{arXiv preprint arXiv:2412.19437}, 2024.

\bibitem[Liu et~al.(2023)Liu, Li, Hall, Liang, and Ma]{liu2023sophia}
Hong Liu, Zhiyuan Li, David Hall, Percy Liang, and Tengyu Ma.
\newblock Sophia: A scalable stochastic second-order optimizer for language model pre-training.
\newblock \emph{arXiv preprint arXiv:2305.14342}, 2023.

\bibitem[Liu et~al.(2025)Liu, Su, Yao, Jiang, Lai, Du, Qin, Xu, Lu, Yan, et~al.]{liu2025muon}
Jingyuan Liu, Jianlin Su, Xingcheng Yao, Zhejun Jiang, Guokun Lai, Yulun Du, Yidao Qin, Weixin Xu, Enzhe Lu, Junjie Yan, et~al.
\newblock Muon is scalable for {LLM} training.
\newblock \emph{arXiv preprint arXiv:2502.16982}, 2025.

\bibitem[Loshchilov and Hutter(2016)]{loshchilov2016sgdr}
Ilya Loshchilov and Frank Hutter.
\newblock Sgdr: Stochastic gradient descent with warm restarts.
\newblock \emph{arXiv preprint arXiv:1608.03983}, 2016.

\bibitem[Loshchilov and Hutter(2019)]{loshchilov2018decoupled}
Ilya Loshchilov and Frank Hutter.
\newblock Decoupled weight decay regularization.
\newblock In \emph{ICLR}, 2019.

\bibitem[Malladi et~al.(2022)Malladi, Lyu, Panigrahi, and Arora]{malladi2022sdes}
Sadhika Malladi, Kaifeng Lyu, Abhishek Panigrahi, and Sanjeev Arora.
\newblock On the sdes and scaling rules for adaptive gradient algorithms.
\newblock \emph{Advances in Neural Information Processing Systems}, 2022.

\bibitem[Nguyen and Salazar(2019)]{nguyen2019transformers}
Toan~Q Nguyen and Julian Salazar.
\newblock Transformers without tears: Improving the normalization of self-attention.
\newblock \emph{arXiv preprint arXiv:1910.05895}, 2019.

\bibitem[Noci et~al.(2022)Noci, Anagnostidis, Biggio, Orvieto, Singh, and Lucchi]{noci2022signal}
Lorenzo Noci, Sotiris Anagnostidis, Luca Biggio, Antonio Orvieto, Sidak~Pal Singh, and Aurelien Lucchi.
\newblock Signal propagation in transformers: Theoretical perspectives and the role of rank collapse.
\newblock \emph{Advances in Neural Information Processing Systems}, 2022.

\bibitem[Orvieto et~al.(2022)Orvieto, Kohler, Pavllo, Hofmann, and Lucchi]{orvieto2022vanishing}
Antonio Orvieto, Jonas Kohler, Dario Pavllo, Thomas Hofmann, and Aur{\'e}lien Lucchi.
\newblock Vanishing curvature in randomly initialized deep relu networks.
\newblock In \emph{AISTATS}, pages 7942--7975, 2022.

\bibitem[Pascanu et~al.(2013)Pascanu, Mikolov, and Bengio]{pascanu2013difficulty}
Razvan Pascanu, Tomas Mikolov, and Yoshua Bengio.
\newblock On the difficulty of training recurrent neural networks.
\newblock In \emph{ICML}, 2013.

\bibitem[Penedo et~al.(2024)Penedo, Kydl{\'\i}{\v{c}}ek, allal, Lozhkov, Mitchell, Raffel, Werra, and Wolf]{penedo2024the}
Guilherme Penedo, Hynek Kydl{\'\i}{\v{c}}ek, Loubna~Ben allal, Anton Lozhkov, Margaret Mitchell, Colin Raffel, Leandro~Von Werra, and Thomas Wolf.
\newblock The fineweb datasets: Decanting the web for the finest text data at scale.
\newblock In \emph{The Thirty-eight Conference on Neural Information Processing Systems Datasets and Benchmarks Track}, 2024.
\newblock URL \url{https://openreview.net/forum?id=n6SCkn2QaG}.

\bibitem[Pethick et~al.(2025)Pethick, Xie, Antonakopoulos, Zhu, Silveti-Falls, and Cevher]{pethick2025training}
Thomas Pethick, Wanyun Xie, Kimon Antonakopoulos, Zhenyu Zhu, Antonio Silveti-Falls, and Volkan Cevher.
\newblock Training deep learning models with norm-constrained lmos.
\newblock \emph{arXiv preprint arXiv:2502.07529}, 2025.

\bibitem[Radford et~al.(2019)Radford, Wu, Child, Luan, Amodei, Sutskever, et~al.]{radford2019language}
Alec Radford, Jeffrey Wu, Rewon Child, David Luan, Dario Amodei, Ilya Sutskever, et~al.
\newblock Language models are unsupervised multitask learners.
\newblock \emph{OpenAI blog}, 1\penalty0 (8):\penalty0 9, 2019.

\bibitem[Reddi et~al.(2018)Reddi, Kale, and Kumar]{reddi2018convergence}
Sashank~J Reddi, Satyen Kale, and Sanjiv Kumar.
\newblock On the convergence of adam and beyond.
\newblock In \emph{International Conference on Learning Representations}, 2018.

\bibitem[Shah et~al.(2025)Shah, Polloreno, Stratos, Monk, Chaluvaraju, Hojel, Ma, Thomas, Tanwer, Shah, et~al.]{shah2025practical}
Ishaan Shah, Anthony~M Polloreno, Karl Stratos, Philip Monk, Adarsh Chaluvaraju, Andrew Hojel, Andrew Ma, Anil Thomas, Ashish Tanwer, Darsh~J Shah, et~al.
\newblock Practical efficiency of muon for pretraining.
\newblock \emph{arXiv preprint arXiv:2505.02222}, 2025.

\bibitem[Shazeer(2020)]{shazeer2020glu}
Noam Shazeer.
\newblock Glu variants improve transformer.
\newblock \emph{arXiv preprint arXiv:2002.05202}, 2020.

\bibitem[Simsekli et~al.(2019)Simsekli, Sagun, and Gurbuzbalaban]{simsekli2019tail}
Umut Simsekli, Levent Sagun, and Mert Gurbuzbalaban.
\newblock A tail-index analysis of stochastic gradient noise in deep neural networks.
\newblock In \emph{ICML}, 2019.

\bibitem[Soboleva et~al.(2023)Soboleva, Al-Khateeb, Myers, Steeves, Hestness, and Dey]{cerebras2023slimpajama}
Daria Soboleva, Faisal Al-Khateeb, Robert Myers, Jacob~R Steeves, Joel Hestness, and Nolan Dey.
\newblock {SlimPajama: A 627B token cleaned and deduplicated version of RedPajama}, 2023.
\newblock URL \url{https://huggingface.co/datasets/cerebras/SlimPajama-627B}.

\bibitem[Su et~al.(2024)Su, Ahmed, Lu, Pan, Bo, and Liu]{su2024roformer}
Jianlin Su, Murtadha Ahmed, Yu~Lu, Shengfeng Pan, Wen Bo, and Yunfeng Liu.
\newblock Roformer: Enhanced transformer with rotary position embedding.
\newblock \emph{Neurocomputing}, 568:\penalty0 127063, 2024.

\bibitem[Tieleman and Hinton(2012)]{tieleman2012lecture}
Tijmen Tieleman and Geoffrey Hinton.
\newblock Lecture 6.5-rmsprop, coursera: Neural networks for machine learning.
\newblock \emph{University of Toronto, Technical Report}, 6, 2012.

\bibitem[Tomihari and Sato(2025)]{tomihari2025understanding}
Akiyoshi Tomihari and Issei Sato.
\newblock Understanding why adam outperforms sgd: Gradient heterogeneity in transformers.
\newblock \emph{arXiv preprint arXiv:2502.00213}, 2025.

\bibitem[Touvron et~al.(2023)Touvron, Lavril, Izacard, Martinet, Lachaux, Lacroix, Rozi{\`e}re, Goyal, Hambro, Azhar, et~al.]{touvron2023llama}
Hugo Touvron, Thibaut Lavril, Gautier Izacard, Xavier Martinet, Marie-Anne Lachaux, Timoth{\'e}e Lacroix, Baptiste Rozi{\`e}re, Naman Goyal, Eric Hambro, Faisal Azhar, et~al.
\newblock Llama: Open and efficient foundation language models.
\newblock \emph{arXiv preprint arXiv:2302.13971}, 2023.

\bibitem[Vaswani et~al.(2017)Vaswani, Shazeer, Parmar, Uszkoreit, Jones, Gomez, Kaiser, and Polosukhin]{vaswani2017attention}
Ashish Vaswani, Noam Shazeer, Niki Parmar, Jakob Uszkoreit, Llion Jones, Aidan~N Gomez, {\L}ukasz Kaiser, and Illia Polosukhin.
\newblock Attention is all you need.
\newblock \emph{Advances in neural information processing systems}, 30, 2017.

\bibitem[Vyas et~al.(2025)Vyas, Morwani, Zhao, Shapira, Brandfonbrener, Janson, and Kakade]{vyas2025soap}
Nikhil Vyas, Depen Morwani, Rosie Zhao, Itai Shapira, David Brandfonbrener, Lucas Janson, and Sham~M. Kakade.
\newblock {SOAP}: Improving and stabilizing shampoo using adam for language modeling.
\newblock In \emph{ICLR}, 2025.

\bibitem[Wang and Komatsuzaki(2022)]{wang2022gpt}
Ben Wang and Aran Komatsuzaki.
\newblock Gpt-j-6b: A 6 billion parameter autoregressive language model. 2021.
\newblock \emph{URL https://github. com/kingoflolz/mesh-transformer-jax}, page~8, 2022.

\bibitem[Wilson et~al.(2017)Wilson, Roelofs, Stern, Srebro, and Recht]{wilson2017marginal}
Ashia~C Wilson, Rebecca Roelofs, Mitchell Stern, Nati Srebro, and Benjamin Recht.
\newblock The marginal value of adaptive gradient methods in machine learning.
\newblock \emph{Advances in neural information processing systems}, 30, 2017.

\bibitem[Xie and Li(2024)]{xie2024implicit}
Shuo Xie and Zhiyuan Li.
\newblock Implicit bias of adamw: $\ell_\infty$-norm constrained optimization.
\newblock In \emph{ICML}, 2024.

\bibitem[Xiong et~al.(2020)Xiong, Yang, He, Zheng, Zheng, Xing, Zhang, Lan, Wang, and Liu]{xiong2020layer}
Ruibin Xiong, Yunchang Yang, Di~He, Kai Zheng, Shuxin Zheng, Chen Xing, Huishuai Zhang, Yanyan Lan, Liwei Wang, and Tieyan Liu.
\newblock On layer normalization in the transformer architecture.
\newblock In \emph{International conference on machine learning}, pages 10524--10533. PMLR, 2020.

\bibitem[Yuan and Gao(2020)]{yuan2020eadam}
Wei Yuan and Kai-Xin Gao.
\newblock Eadam optimizer: How $\epsilon$ impact adam.
\newblock \emph{arXiv preprint arXiv:2011.02150}, 140, 2020.

\bibitem[Zhang and Sennrich(2019)]{zhang2019root}
Biao Zhang and Rico Sennrich.
\newblock Root mean square layer normalization.
\newblock \emph{Advances in Neural Information Processing Systems}, 32, 2019.

\bibitem[Zhang et~al.(2025)Zhang, Morwani, Vyas, Wu, Zou, Ghai, Foster, and Kakade]{zhang2025how}
Hanlin Zhang, Depen Morwani, Nikhil Vyas, Jingfeng Wu, Difan Zou, Udaya Ghai, Dean Foster, and Sham~M. Kakade.
\newblock How does critical batch size scale in pre-training?
\newblock In \emph{ICLR}, 2025.

\bibitem[Zhang et~al.(2020{\natexlab{a}})Zhang, He, Sra, and Jadbabaie]{zhang2020why}
Jingzhao Zhang, Tianxing He, Suvrit Sra, and Ali Jadbabaie.
\newblock Why gradient clipping accelerates training: A theoretical justification for adaptivity.
\newblock In \emph{ICLR}, 2020{\natexlab{a}}.

\bibitem[Zhang et~al.(2020{\natexlab{b}})Zhang, Karimireddy, Veit, Kim, Reddi, Kumar, and Sra]{zhang2020adaptive}
Jingzhao Zhang, Sai~Praneeth Karimireddy, Andreas Veit, Seungyeon Kim, Sashank Reddi, Sanjiv Kumar, and Suvrit Sra.
\newblock Why are adaptive methods good for attention models?
\newblock \emph{Advances in Neural Information Processing Systems}, 33:\penalty0 15383--15393, 2020{\natexlab{b}}.

\bibitem[Zhang et~al.(2022)Zhang, Chen, Shi, Sun, and Luo]{zhang2022adam}
Yushun Zhang, Congliang Chen, Naichen Shi, Ruoyu Sun, and Zhi-Quan Luo.
\newblock Adam can converge without any modification on update rules.
\newblock \emph{Advances in Neural Information Processing Systems}, 2022.

\bibitem[Zhang et~al.(2024{\natexlab{a}})Zhang, Chen, Ding, Li, Sun, and Luo]{zhang2024why}
Yushun Zhang, Congliang Chen, Tian Ding, Ziniu Li, Ruoyu Sun, and Zhi-Quan Luo.
\newblock Why transformers need adam: A hessian perspective.
\newblock In \emph{Neural Information Processing Systems}, 2024{\natexlab{a}}.

\bibitem[Zhang et~al.(2024{\natexlab{b}})Zhang, Chen, Li, Ding, Wu, Kingma, Ye, Luo, and Sun]{zhang2024adam}
Yushun Zhang, Congliang Chen, Ziniu Li, Tian Ding, Chenwei Wu, Diederik~P Kingma, Yinyu Ye, Zhi-Quan Luo, and Ruoyu Sun.
\newblock Adam-mini: Use fewer learning rates to gain more, 2024{\natexlab{b}}.

\bibitem[Zhao et~al.(2025)Zhao, Morwani, Brandfonbrener, Vyas, and Kakade]{zhao2025deconstructing}
Rosie Zhao, Depen Morwani, David Brandfonbrener, Nikhil Vyas, and Sham~M Kakade.
\newblock Deconstructing what makes a good optimizer for autoregressive language models.
\newblock In \emph{ICLR}, 2025.

\bibitem[Zucchet and Orvieto(2024)]{zucchet2024recurrent}
Nicolas Zucchet and Antonio Orvieto.
\newblock Recurrent neural networks: vanishing and exploding gradients are not the end of the story.
\newblock \emph{Advances in Neural Information Processing Systems}, 2024.

\end{thebibliography}

\newpage

\appendix

\tableofcontents 
\section{Experimental details}
\label{app:exp-details}

For pre-training Transformers on Causal Language Modeling, we build upon the nanoGPT~\citep{karpathy2022nanogpt} implementation,
augmenting it with Rotational Positional Embedding~\citep{su2024roformer}, RMSNorm~\citep{zhang2019root}, and SwiGLU~\citep{shazeer2020glu}. All our models have a vocabulary size of 50280 and make use of GPT-Neox tokenizer~\citep{black2022gpt}. We adopt an enhanced training recipe, made popular by large language models such as LLaMa~\citep{touvron2023llama}. These modifications include:
training in bfloat16; employing a linear learning rate warm-up for 10\% of the training steps, followed by cosine annealing
to $1e-5$. Global norm clipping is used~(unless specified or ablated upon) for gradients with norms above
1~(on the raw gradient, as a first step). We have no weight tying between the embedding and the last linear layer. We always report validation perplexity on a separate subset 
consisting of 100M tokens. Seeds, when provided, are relative to distinct network initialization.

\paragraph{Computational Resources.} All our experiments at a 160M parameter scale are performed on a single NVIDIA A100-SXM4-80GB. At compute optimality~(most of our experiments) each run takes approximately 5.83 hours. Our runs at a 410M parameter scale are performed on 8 NVIDIA A100-SXM4-80GB GPUs, and each run here takes approximately 4.83 hours. For all our runs, we fill up memory and optimize to minimize the gradient accumulation steps~(usually, around 8). 

\paragraph{Code.} All our runs use the repository
\begin{center}
    \url{https://github.com/Niccolo-Ajroldi/plainLM}
\end{center}

\paragraph{Model settings (160M).} We use the same configuration as~\citep{biderman2023pythia}: \url{https://github.com/EleutherAI/pythia/blob/main/models/160M/pythia-160m.yml}

\begin{itemize}
    \item \textit{Layers:} 12 Transformer~\citep{vaswani2017attention} layers
    \item \textit{Attention heads:} 12
    \item \textit{Hidden size:} 768
    \item \textit{Attention implementation:} Flashattention~\citep{dao2022flashattention}.
    \item \textit{MLP type:} SwiGLU~\citep{shazeer2020glu} with expansion factor 8/3.
    \item \textit{Backbone:} PreLN transformer~\citep{xiong2020layer} with skip connections.
    \item \textit{Normalization:} RMSnorm~\citep{zhang2019root} for both Attention and MLP.
    \item \textit{Position embeddings:} Rotary embeddings (RoPE) to 25\% of dimensions~(\citep{su2024roformer})
    \item \textit{Initialization:} the MLP and Attention output weights are initialized with variance $0.02/\sqrt{2 \# \text{layers}}$~(scaling also similar to~\citep{radford2019language}). All other weights~(comprising embeddings) are initialized with a standard deviation of $0.02$~(\citet{nguyen2019transformers,wang2022gpt}, Sec. 2.2). Biases are always initialized at zero.
    \item \textit{Precision:} Mixed precision FP16 enabled.
    \item \textit{Dropout:} Disabled for both hidden and attention layers~(see also~\citet{chowdhery2023palm}).
\end{itemize}

\paragraph{Model settings (410 M).} We use the same setting as~\citep{biderman2023pythia}, configuration can be found here: \url{https://github.com/EleutherAI/pythia/blob/main/models/410M/pythia-410m-deduped.yml}

\begin{itemize}
    \item \textit{Layers:} 24 Transformer layers
    \item \textit{Attention heads:} 16
    \item \textit{Hidden size:} 1024
    \item Other settings as 160M parameters.
\end{itemize}

\subsection{Experiments on {\color{magenta}SlimPajama -- 160M} parameters model}
\label{app:setting_160_sp}
On the Cerebras SlimPajama-627B~\citep{cerebras2023slimpajama} dataset: \url{https://huggingface.co/datasets/cerebras/SlimPajama-627B} at a 160M scale we present three experimental sections:
\begin{itemize}
    \item Section~\ref{app:setting_160_sp_2048} -- core setting, ablating on \textbf{all optimizers}.
    \item Section~\ref{app:setting_160_sp_512} -- ablating on a \textbf{smaller sequence length}.
    \item Section~\ref{app:setting_160_sp_2048_batch} -- ablating at \textbf{different batch sizes}.    
\end{itemize}

\subsubsection{Sequence Length 2048, Batch size 256, 3.2 B Tokens ~(6200 gradient steps)}
\label{app:setting_160_sp_2048}

This setup comprises a total of 747 full training runs. We always use warm-up~(10\%) and cosine anneal until a learning rate of $1e-5$. This setting is Chinchilla-optimal~(20 tokens/parameter).

$\lambda$ here denotes the weight decay, always decoupled~\citep{loshchilov2018decoupled}.

\paragraph{Core experiments:} These are the core experimental results for this setting.
\begin{itemize}
    \item \textbf{SGD} with momentum $\beta$ and \textbf{global norm clipping} to $1$ (Gclip, dampening to $1-\beta$)\\ --- \textit{84 full runs}~(Figure~\ref{fig:160_sgd_clip}, top).
    \begin{align*}
        (\eta,\beta,\lambda)&\in[2.0, 1.0, 0.5, 0.25, 0.125, 0.0625, 0.03125]\\  & \quad \times[0.9875, 0.975, 0.95, 0.9]\\ &\quad \times [0, 1e-3, 1e-4].
    \end{align*}

    \item \textbf{SGD} with momentum $\beta$ with (1) \textbf{global norm clipping} of raw gradient to $1$ (Gclip) and (2) \textbf{coordinate clipping} (Cclip) to $1$ after momentum is applied. Dampening is set to $1-\beta$, $\lambda$~(weight decay) is set to 0, as the previous point revealed decreasing performance on SGD\\ --- \textit{24 full runs}~(Figure~\ref{fig:160_sgd_clip}, bottom).
    \begin{align*}
        (\eta,\beta,\lambda)&\in[2.0, 1.0, 0.5, 0.25, 0.125, 0.0625]\\  & \quad \times[0.9875, 0.975, 0.95, 0.9]
    \end{align*}
    \item \textbf{SGD} with momentum $\beta$~(vanilla, dampening to $1-\beta$, \textbf{no clipping}). $\lambda= 0$ (weight decay).\\ --- \textit{28 full runs}~(Figure~\ref{fig:160_sgd})
    \begin{align*}
        (\eta,\beta)&\in[0.25, 0.125, 0.0625, 0.03125, 0.015625, 0.0078125, 0.00390625]\\  & \quad \times[0.9875, 0.975, 0.95, 0.9].
    \end{align*}

    \item \textbf{Adam} with global norm clipping to 1 and with $\lambda = 0.1$~(weight decay) and $\epsilon = 1e-8$ (usual Pytorch setup, see also~\citet{biderman2023pythia}).\\ -- \textit{200 full runs}~(Figure~\ref{fig:big_sweep})
    \begin{align*}
        (\eta,\beta_1,\beta_2)&\in [0.016, 0.008, 0.004, 0.002, 0.001]\\& \quad \times[0.9875, 0.975, 0.95, 0.9, 0.8]\\& \quad \times[0.996875, 0.99375, 0.9875, 0.975, 0.95, 0.9, 0.8, 0.6]
    \end{align*}

    \item \textbf{Adam}\textbf{ without global norm clipping} and with $\lambda = 0.1$~(weight decay) and $\epsilon = 1e-8$ (usual Pytorch setup, see also~\citet{biderman2023pythia}).\\ -- \textit{165 full runs}~(Figure~\ref{fig:adam_no_clip})
    \begin{align*}
        (\eta,\beta_1,\beta_2)&\in [0.032, 0.016, 0.008, 0.004, 0.002, 0.001]\\& \quad \times[0.975, 0.95, 0.9, 0.8]\\& \quad \times[0.9875, 0.975, 0.95, 0.9, 0.8, 0.6]
    \end{align*}

    \item \textbf{RMSprop} implemented using the AdamW Pytorch class using $\beta_1=0$. We again use $\lambda = 0.1$~(weight decay) and $\epsilon = 1e-8$. \\
    -- \textit{48 full runs}~(Figure~\ref{fig:160_RMS}).
    \begin{align*}
        (\eta,\beta_2)&\in [0.004, 0.002, 0.001, 0.0005, 0.00025, 0.000125]\\& \quad \times[0.9875, 0.975, 0.95, 0.9, 0.8, 0.6, 0.4, 0.0]
    \end{align*}

    \item \textbf{Signum} with weight decay $\lambda=0.1$ as also suggested by~\citep{zhao2025deconstructing}~(their Figure 5, top left panel). We \textbf{ablate on presence of global norm gradient clipping}~(to norm 1).\\
    -- \textit{70 full runs}~(Figure~\ref{fig:big_sweep}).
    \begin{align*}
        (\eta,\beta, \text{clip})&\in [0.004, 0.002, 0.001, 0.0005, 0.00025, 0.000125, 0.0000625]\\& \quad \times[0.9875, 0.975, 0.95, 0.9, 0.8]\\& \quad \times[\text{True}, \text{False}]
    \end{align*}
    Note that Signum with and without gradient clipping are two different methods: here, clipped gradients are first averaged and only then the sign is taken. Instead, clipping on the EMA of signed gradients (next method) should have no effect (apart from non-determinism).

    \item \textbf{EMASign} with weight decay $\lambda=0.1$. We ablate on the presence of global norm gradient clipping~(to norm 1) \textit{out of mistake}: the two methods are equal!\\
    -- \textit{70 runs (35 duplicate runs)}~(Figure~\ref{fig:160_EMA_sign})
    \begin{align*}
        (\eta,\beta, \text{clip})&\in [0.001, 0.0005, 0.00025, 0.000125, 0.0000625, 0.00003125, 0.000015625]\\& \quad \times[0.9875, 0.975, 0.95, 0.9, 0.8]\\& \quad \times[\text{True}, \text{False}]
    \end{align*}
\end{itemize}

\paragraph{Ablations:} These ablations were performed to test side-claims in the paper.

\begin{itemize}
    \item \textbf{Adam} with global norm clipping to 1 and $\lambda=0.1, \beta_1=\beta_2=0.95, \eta = 0.008$~(best setup from Figure~\ref{fig:big_sweep}). We report performance for 3 seeds using different $\epsilon$ values.\\
    -- \textit{18 full runs}~(Table~\ref{tab:ablation_eps}).
    $$\epsilon\in [1e-3,1e-6, 1e-8, 1e-10, 1e-12, 1e-15],$$
    and influence of initializing exponential moving averages to zero~(default, ZI) or to the stochastic quantity of interest~(gradient initialization, GI). At the same time, we try to remove bias correction. These experiments are presented with 3 random initialization seeds:\\
    -- \textit{9 full runs}~(Table~\ref{tab:ablation_ZI_BC}).
    \item \textbf{Signum} with global norm clipping to 1 and $\lambda=0.1, \beta = 0.95$~(best setting from Figure~\ref{fig:big_sweep}): we ablate on fixed mollifiers for zero-initialized~(ZI) or gradient-initialized~(GI) momentum. The mollified we study is $m_k/(\sqrt{m_k}+\epsilon)$:
    \begin{align*}
        (\eta,\epsilon)&\in [0.001, 0.0005, 0.00025, 0.000125]\times[1e-3, 1e-6, 1e-9]
    \end{align*}    
    -- \textit{12 full runs}~(Table~\ref{tab:ablation_eps}).\\
    We additionally test the influence of ZI vs. GI with three random seeds at $\epsilon=0$.\\
    -- \textit{5 full runs}~(Table~\ref{tab:ablation_ZI_BC}).
\end{itemize}

\paragraph{Other:} for the best-performing variants of core experiments, we initialize the model with two other random seeds. This accounts for \\
-- \textit{14 additional full runs}~(Table~\ref{tab:algorithm-performance}).

\subsubsection{Sequence Length 2048, Batch size 256, 6.4 B Tokens ~(12400 gradient steps)}
\label{app:setting_160_sp_2048_double}

The setup here is exactly as in \S\ref{app:setting_160_sp_2048}, but we train for $2\times$ the token budget. We test our core claim~($\beta_1=\beta_2$ works well), and hence we run:

\begin{itemize}
    \item \textbf{Adam} with global norm clipping to 1 and with $\lambda = 0.1$~(weight decay) and $\epsilon = 1e-8$.\\ -- \textit{168 full runs}~(Figure~\ref{fig:big_sweep_2chinchilla})
    \begin{align*}
        (\eta,\beta_1,\beta_2)&\in [0.032, 0.016, 0.008, 0.004, 0.002, 0.001, 0.0005]\\& \quad \times[0.9875, 0.975, 0.95, 0.9]\\& \quad \times[0.99375, 0.9875, 0.975, 0.95, 0.9, 0.8]
    \end{align*}
\end{itemize}

\subsubsection{Sequence Length 512, Batch size 256, 3.2 B Tokens ~(24800 gradient steps)}
\label{app:setting_160_sp_512}
This setup comprises a total of 55 full training runs. We test our core claims~(Signum underperforms Adam, $\beta_1=\beta_2$ works well) at a smaller sequence length. Setting is exactly the same as \S\ref{app:setting_160_sp_2048} for all methods, unless stated otherwise.
\begin{itemize}
    \item \textbf{Adam}, we limit this ablation to $\beta_1 = 0.95$, 
    $$(\eta,\beta_2)\in [0.001, 0.002, 0.004, 0.008, 0.016]\times[0.99375, 0.9875, 0.975, 0.95, 0.9, 0.8]$$
    -- \textit{25 full runs}~(Figure~\ref{fig:adam_512}).

    \item \textbf{Signum}, we do a full ablation using global norm gradient clipping to 1.
    $$(\eta, \beta) \in [0.0000625, 0.000125, 0.00025, 0.0005, 0.001, 0.002]\times[0.9875, 0.975, 0.95, 0.9, 0.8]$$
    -- \textit{30 full runs}~(Figure~\ref{fig:adam_512}).

\end{itemize}

\subsubsection{Sequence Length 2048, Variable batch size, 2.5 B Tokens}
\label{app:setting_160_sp_2048_batch}
We use here a slightly reduced token budget~(2.5B, 20 tokens for every non-embedding parameter) and run the same Adam tuning experiment presented in Figure~\ref{fig:big_sweep} for batch size 256. We actually run this experiment again at a batch size of 256, and test batch sizes of 128 and 512 reducing or doubling the number of steps accordingly~(same token budget). The sequence length is still 2048, and the dataset SlimPajama. Due to the reduced number of tokens, each run takes approximately 4.7 hours on our hardware. We implement variation of batch size using gradient accumulation $(4,8,16)$ at a micro-batch size of 32 sequences. This setup comprises a total of 500 full training runs.

\textbf{Adam} with $\lambda = 0.1$~(weight decay) and $\epsilon = 1e-8$ (usual setup, see~\citet{biderman2023pythia}), we clip gradients to global norm 1.

\begin{itemize}[itemsep=0pt, leftmargin=*]
    \item For batch size 256:
    \begin{align*}
        (\eta,\beta_1,\beta_2)&\in [0.016, 0.008, 0.004, 0.002, 0.001]\\& \quad \times[0.9875, 0.975, 0.95, 0.9, 0.8]\\& \quad \times[0.996875, 0.99375, 0.9875, 0.975, 0.95, 0.9, 0.8, 0.6]
    \end{align*}
    \item For batch size 128 and 512:
    \begin{align*}
        (\eta,\beta_1,\beta_2)&\in [0.0005, 0.001, 0.0014, 0.002, 0.0028, 0.004, 0.0056, 0.008, 0.0112, 0.016]\\& \quad \times[0.975, 0.95, 0.9]\\& \quad \times 1-[4,2,1,0.5,0.25]\cdot(1-\beta_1) \qquad \text{(i.e. 3 higher and 2 lower values in grid)}
    \end{align*}
    Note that here we overturned the learning rate, the reason for this is the square root scaling law in~\citet{malladi2022sdes,compagnoni2025adaptive}: if batch size scales by 2, learning rate should scale as $\sqrt{2}$. We see in \S\ref{app:more_batches} that this indeed seems to hold true, yet noise prevents us from making precise verification claims.
\end{itemize}
 -- \textit{500 full runs}~(\S\ref{app:more_batches}).


\subsection{Experiments on {\color{magenta}SlimPajama -- 410M} parameters model, 8.2 B tokens}
\label{app:setting_410_sp}
All our experiments here use the Cerebras SlimPajama-627B~\citep{cerebras2023slimpajama} dataset: \url{https://huggingface.co/datasets/cerebras/SlimPajama-627B}. We focus on evaluating whether $\beta_1=\beta_2$ yields good performance in this settings. We scale up the batch size by a factor 2 compared to Section~\ref{app:setting_160_sp}, as suggested by~\citep{zhang2025how}. We perform our experiments at compute optimality~(8.2B tokens, 20 tokens per parameter).

\textbf{Adam} with $\lambda = 0.1$~(weight decay) and $\epsilon = 1e-8$ (usual setup, see~\citet{biderman2023pythia}), we clip gradients to global norm 1:

\begin{itemize}
    \item   $\beta_1 = 0.9$
    \begin{align*}
        (\eta, \beta_2)&\in [0.016, 0.008, 0.004, 0.002]\times[0.95, 0.9, 0.8]
    \end{align*}
    \item   $\beta_1 = 0.95$
    \begin{align*}
        (\eta, \beta_2)&\in [0.016, 0.008, 0.004, 0.002]\times[0.9875, 0.975, 0.95, 0.9]
    \end{align*}
    \item   $\beta_1 = 0.975$
    \begin{align*}
        (\eta, \beta_2)&\in [0.016, 0.008, 0.004, 0.002]\times[0.99375, 0.9875, 0.975, 0.95]
    \end{align*}
\end{itemize}
 -- \textit{44 full runs}~(Figure~\ref{fig:big_sweep410}).

\subsection{Experiments on {\color{magenta}Fineweb -- 160M} parameters model, 3.2B tokens -- no weight decay}
\label{app:setting_160_fw}

While testing our claims on a different dataset, we also crucially \textit{remove weight decay} here. Our setting is otherwise identical to that of \S\ref{app:setting_160_sp_2048}: Sequence length is 2048, batch size is 256, model has 160 parameters and we train on 3.2B tokens from Fineweb~\citep{penedo2024the}~\url{https://huggingface.co/datasets/HuggingFaceFW/fineweb}.

\begin{itemize}
    \item \textbf{Adam} with $\lambda = 0$~(no weight decay!) and $\epsilon = 1e-8$ (usual setup, see~\citet{biderman2023pythia}). We clip gradients to global norm 1.\\ 
    \begin{align*}
        (\eta,\beta_1,\beta_2)&\in [0.032, 0.016, 0.008, 0.004, 0.002, 0.001]\\& \quad \times[0.975, 0.95, 0.9]\\& \quad \times[0.9875, 0.975, 0.95, 0.9, 0.8]
    \end{align*}
    -- \textit{90 full runs}~(Figure~\ref{fig:160_fineweb})

    \item \textbf{Signum} with $\lambda=0$ (no weight decay) as also suggested by~\citep{zhao2025deconstructing}~(Figure 5, top left panel). We clip gradients to global norm 1.
    \begin{align*}
        (\eta,\beta)&\in [0.004, 0.002, 0.001, 0.0005, 0.0000625, 0.00025, 0.000125]\\ & \quad \times[0.975, 0.95, 0.9]
    \end{align*}
    -- \textit{24 full runs}~(Figure~\ref{fig:160_fineweb}).
\end{itemize}

\section{Complementary Experimental Results}
\label{app:more_experiments}

The results in this section complement the discussion in \S\ref{sec:exp}. We organize them in 5 subsections, and report all technical details in \S\ref{app:exp-details}.

\begin{itemize}
    \item \S\ref{app:tuning_other_methods} outlines all hyperparameter tuning curves for the setting in Table~\ref{tab:algorithm-performance} for \SGD{}~(with/without clipping and with/without weight decay) -- Figure~\ref{fig:160_sgd_clip} and~\ref{fig:160_sgd}, \RMSprop{} without momentum -- Figure~\ref{fig:160_RMS}, and momentum on top of \SignSGD{} -- Figure~\ref{fig:160_EMA_sign}.
    \item \S\ref{app:shorter_seq} validates that $\beta_1=\beta_2$ is a strong-performing option for \Adam{} at a shorter sequence length. Here, we also show that \Signum{} performance is still suboptimal~(cf. Figure~\ref{fig:big_sweep}).
    \item \S\ref{app:more_batches} validates that $\beta_1=\beta_2$ is a strong-performing option for \Adam{} across different batchsizes. This data, comprising training 500 models, is summarized in Figure~\ref{fig:correlation}.
    \item \S\ref{app:fineweb} reproduces the \Signum{}-\Adam{} gap on Fineweb~\citep{penedo2024the}. Compared to Figure~\ref{fig:big_sweep} and the settings above, \textit{here we compare at zero weight decay to eliminate this additional confounder}.
    \item \S\ref{app:sanity_checks} confirms on the validity of our findings when ablating on nuances of \Signum{} and \Adam{} such as initialization and bias correction. These findings complement \S\ref{sec:sanity_checks}.
\end{itemize}

\subsection{Tuning for Table~\ref{tab:algorithm-performance}}
\label{app:tuning_other_methods}
\vspace{-1mm}
\paragraph{Setup Summary.} 160 M parameters LM on SlimPajama, trained for 3.2 B tokens at a batchsize of 256 $\times$ 2048 sequence length.

\vspace{-3mm}
\paragraph{Comment.} Our objective here is to tune to best, despite the combinatorially exploding number of options, our methods in Table~\ref{tab:algorithm-performance}. Details regarding our hyperparameters grid and model configurations are reported in \S\ref{app:exp-details}. We remind that tuning for \Signum{} and \Adam{} is presented directly in the main paper as Figure~\ref{fig:big_sweep}. \textbf{All figures below show optimal tuning jointly in learning rate and momentum space}. While tuning for \RMSprop{} and momentum on \SignSGD{} is straightforward, \SGD{} requires more attention: we found that removing weight decay was always beneficial when global norm clipping the raw gradient, hence we adopt this option also for the non-clipped variant, and for the variant that includes an additional coordinate clipping step after applying momentum. We believe this is due to the decoupled nature of weight decay, combined with the high learning rates required for good performance in \SGD{}. 

\vspace{-3mm}
\paragraph{Finalizing Table~\ref{tab:algorithm-performance}.} After careful tuning, we select for each method the best configuration~(given by figures below) and run two additional seeds to report final results with 2-sigma confidence bars.

\begin{figure}[ht]
\includegraphics[width=0.99\linewidth]{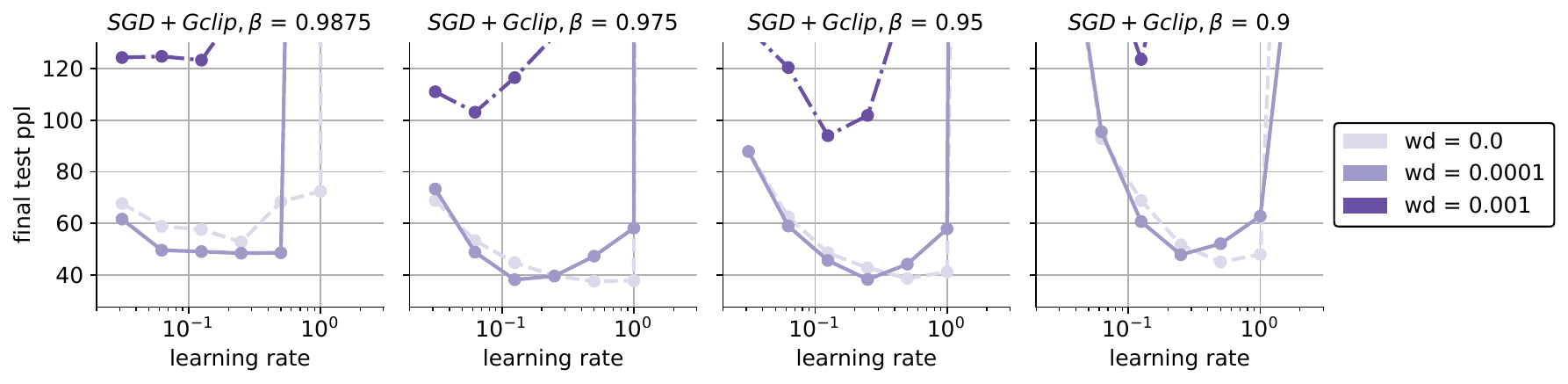}
\includegraphics[width=0.95\linewidth]{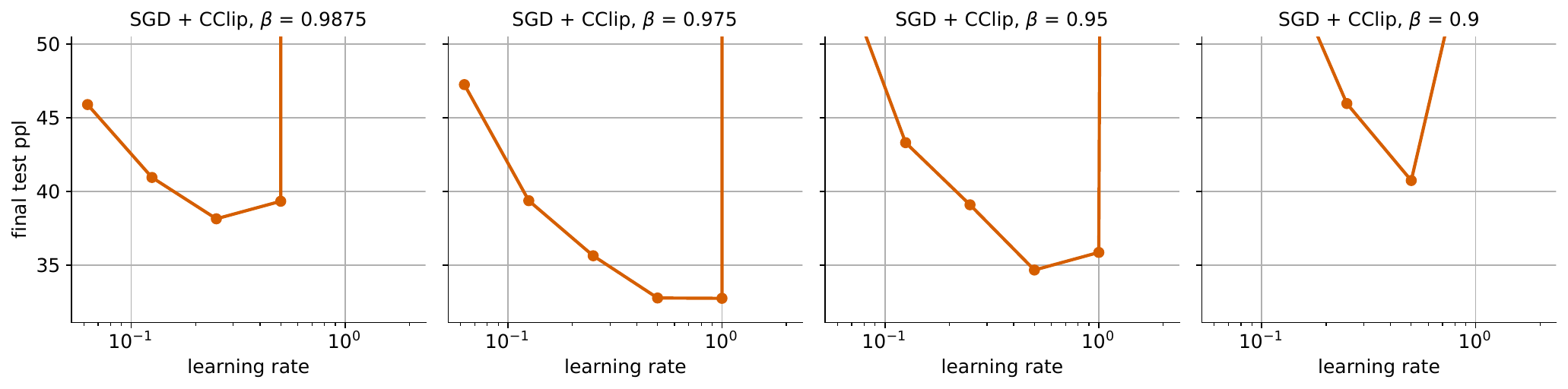}
    \caption{\small \textit{\textbf{(top) SGD with global norm clipping}. We found it beneficial to remove weight decay: the best setting achieves 37.53 ppl, while a slightly larger wd leads to 38.11. a weights decay of 0.001 is too large and yields 93.7 best validation perplexity. \textbf{(bottom) SGD with global norm clipping on raw gradients, followed by coordinate clipping on momentum.} We remove weight decay as suggested by the top plot. We observe an improvement of 5 perplexity points.}}
    \label{fig:160_sgd_clip}
\end{figure}

\begin{figure}[ht]
    \centering
\includegraphics[width=0.9\linewidth]{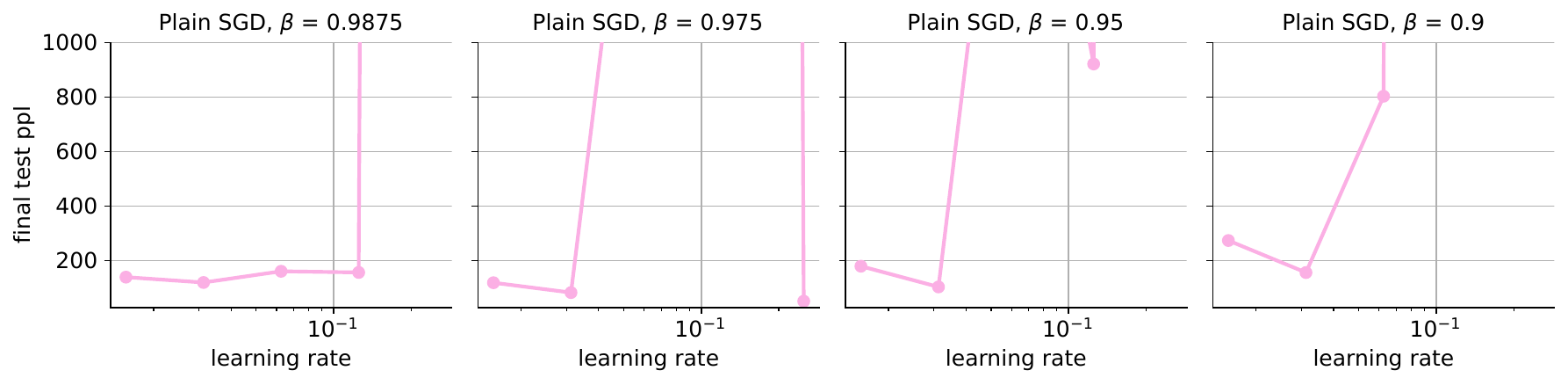}
    \caption{\small \textit{\textbf{SGD without coordinate-wise clipping} at zero weight decay~(as suggested by Figure~\ref{fig:160_sgd_clip} ). }}
    \label{fig:160_sgd}
\end{figure}

\begin{figure}[ht]
    \centering
\includegraphics[width=0.99\linewidth]{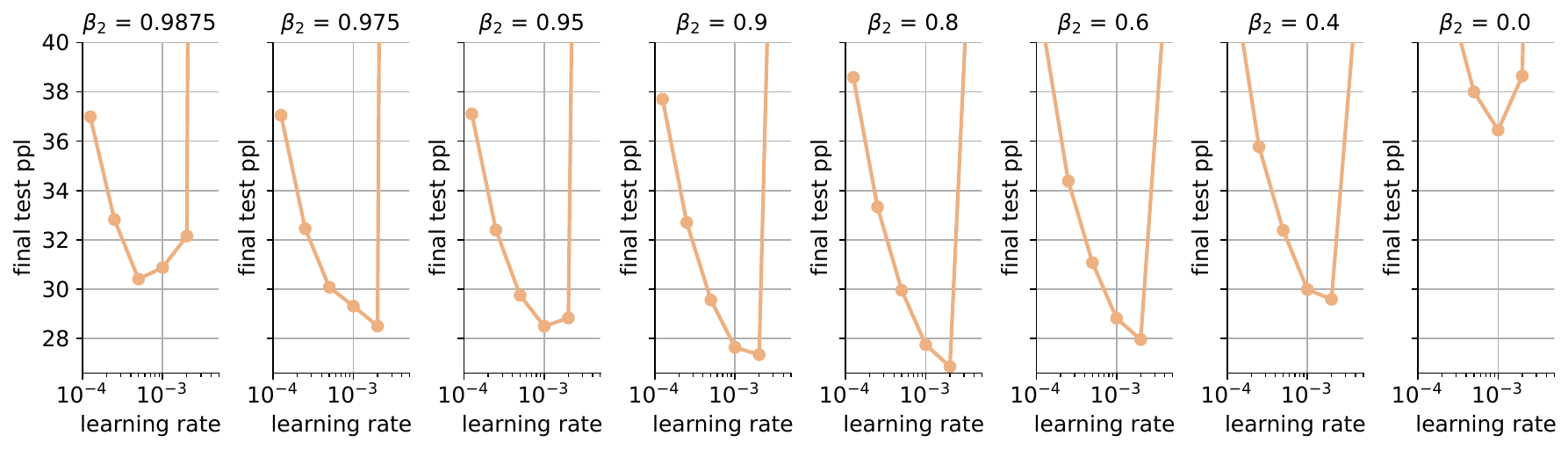}
    \caption{\small \textit{\textbf{RMSprop with decoupled weight decay} 0.1. Implemented with Pytorch AdamW setting $\beta_1=0$.}}
    \label{fig:160_RMS}
\end{figure}

\begin{figure}[ht]
    \centering
\includegraphics[width=0.9\linewidth]{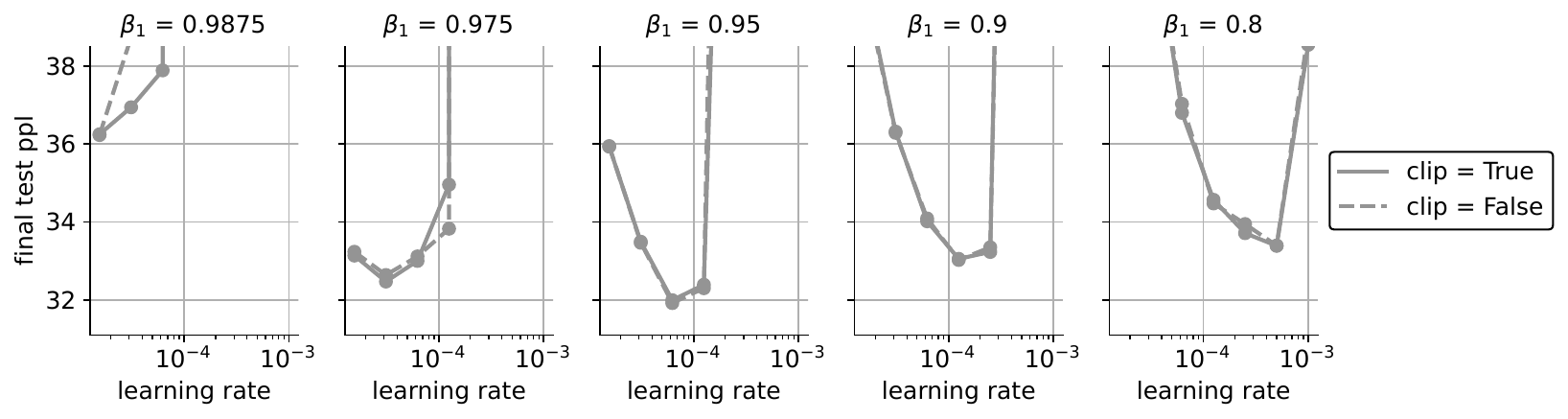}
    \caption{\small \textit{\textbf{Momentum on SignSGD} with decoupled weight decay. We implement this just for completeness to show that it is performing worse than \Signum. Clipping has mathematically no effect~(we did not notice at first, so we show the result anyways).}}
    \label{fig:160_EMA_sign}
\end{figure}

\clearpage
\newpage

\begin{figure}[ht]
    \centering
\includegraphics[width=0.99\linewidth]{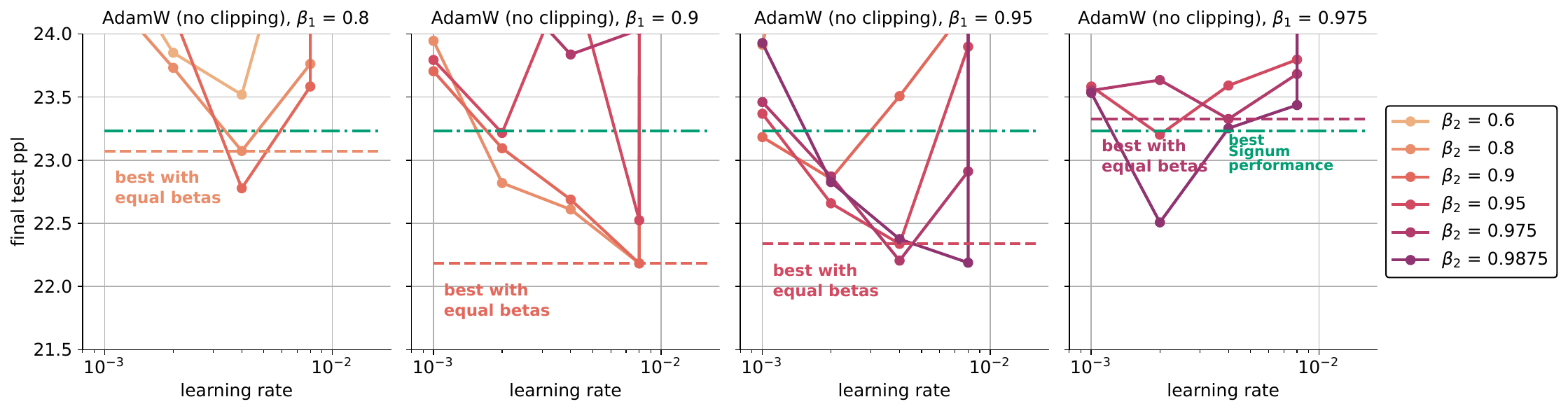}
    \caption{\small \textit{\textbf{AdamW without global norm clipping on gradients} with decoupled weight decay. Compared to Figure~\ref{fig:big_sweep}, here we do not clip gradients as a first preprocessing step. Performance is slightly worse, and results are noisier. The best setting, among the ones we tried, is $\beta_1=\beta_2=0.9$. Note, however, that for large/small $\beta_1$s, we observe that some specific configuration with high $\beta_2$ can be beneficial~(while still suboptimal if $\beta_1=\beta_2$ is tuned). In practice, best performance can also be achieved in this setting by merely tuning $\beta_1=\beta_2=\beta$, resulting in drastic hyperparameter grid size reduction.}}
    \label{fig:adam_no_clip}
\end{figure}

\subsection{Effect of More Training Tokens in Figure~\ref{fig:big_sweep}}
\label{app:more_tokens}
We run part of the experiments in Figure~\ref{fig:big_sweep} at twice the token budget. Results are conceptually very similar, and show that, on top of $\beta_1=\beta_2$ being a performance choice for AdamW, that there exists a strong correlation between $\beta$ values~(see Fig.~\ref{fig:correlation}).

\begin{figure}[ht]
    \centering
\includegraphics[width=0.99\linewidth]{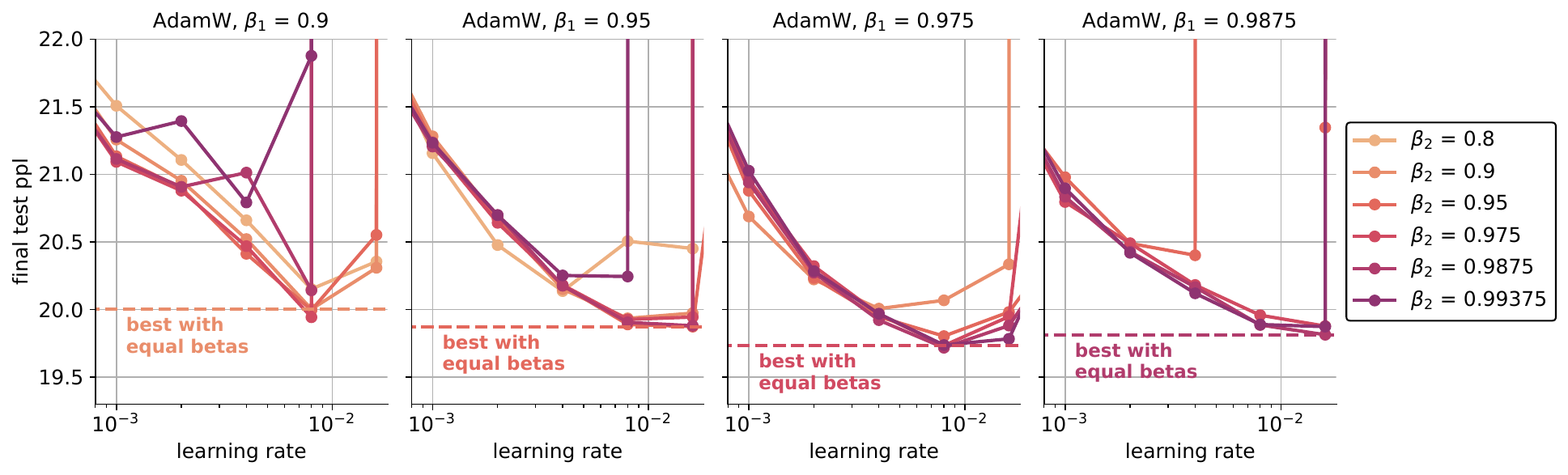}
    \caption{\small \textit{\textbf{AdamW}, same setting as Figure~\ref{fig:big_sweep}, but trained for twice the number of tokens.}}
    \label{fig:big_sweep_2chinchilla}
\end{figure}

\newpage
\subsection{Effect of Shorter Sequence Length in Figure~\ref{fig:big_sweep}}
\label{app:shorter_seq}
We run part of the experiments in Figure~\ref{fig:big_sweep} at a lower sequence length~(512), for a batch size of 256 sequences~(as Figure~\ref{fig:big_sweep}). The model here still sees 3.2B tokens~(compute optimal), but number of effective optimizer steps is 4 times bigger compared to the 2048 sequence length setting. While we still observe a sizeable gap between \Signum{} and \Adam{}, we note that this is smaller compared to Figure~\ref{fig:big_sweep}, as noted also by~\citet{zhao2025deconstructing} in a similar setting.

\begin{figure}[ht]
    \centering
\includegraphics[width=0.9\linewidth]{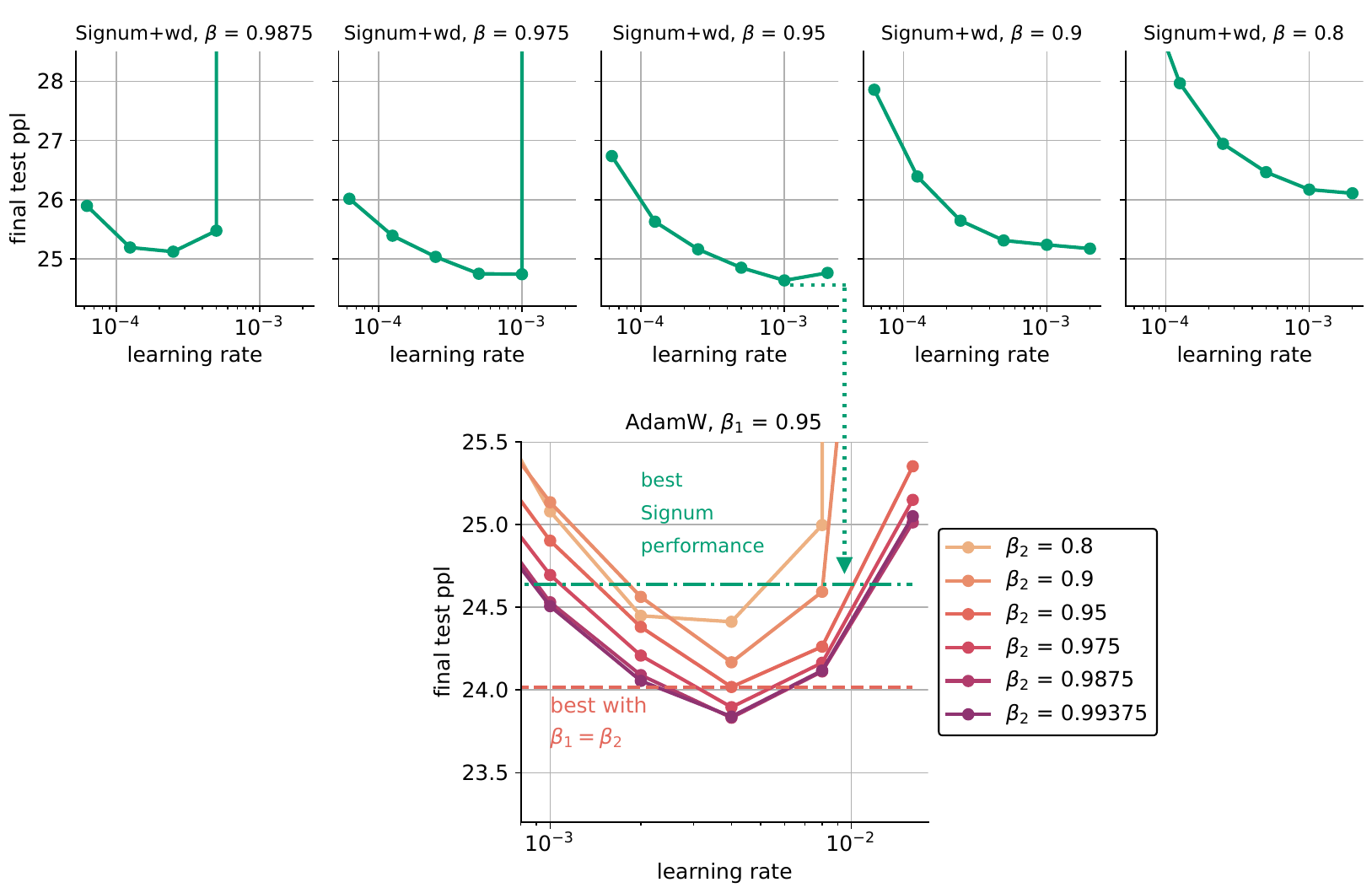}
    \caption{\small \textit{\textbf{AdamW vs Signum}, same setting as Figure~\ref{fig:big_sweep}, but at a smaller sequence length~(512).}}
    \label{fig:adam_512}
\end{figure}


\clearpage
\newpage
\subsection{Batch size ablation for Figure~\ref{fig:big_sweep}}
\label{app:more_batches}

We run part of the experiments in Figure~\ref{fig:big_sweep} at a lower and higher batch size. All other details remain the same and are summarized in \S\ref{app:exp-details} -- except for the number of steps performed: due to limitations in our resources, we chose here to train models for 2.5B tokens -- i.e. a slightly undertrained setting~(optimal would be 3.2B). In line with~\citep{malladi2022sdes,compagnoni2025adaptive} we consider half-steps when tuning. All experiments use a weight decay of 0.1.

Despite some imperfections and noise in performance, we notice that $\beta_1=\beta_2$ is a strong choice even at different batch sizes, our Takeaway 2. 

\begin{figure}[ht]
    \centering
\includegraphics[width=0.95\linewidth]{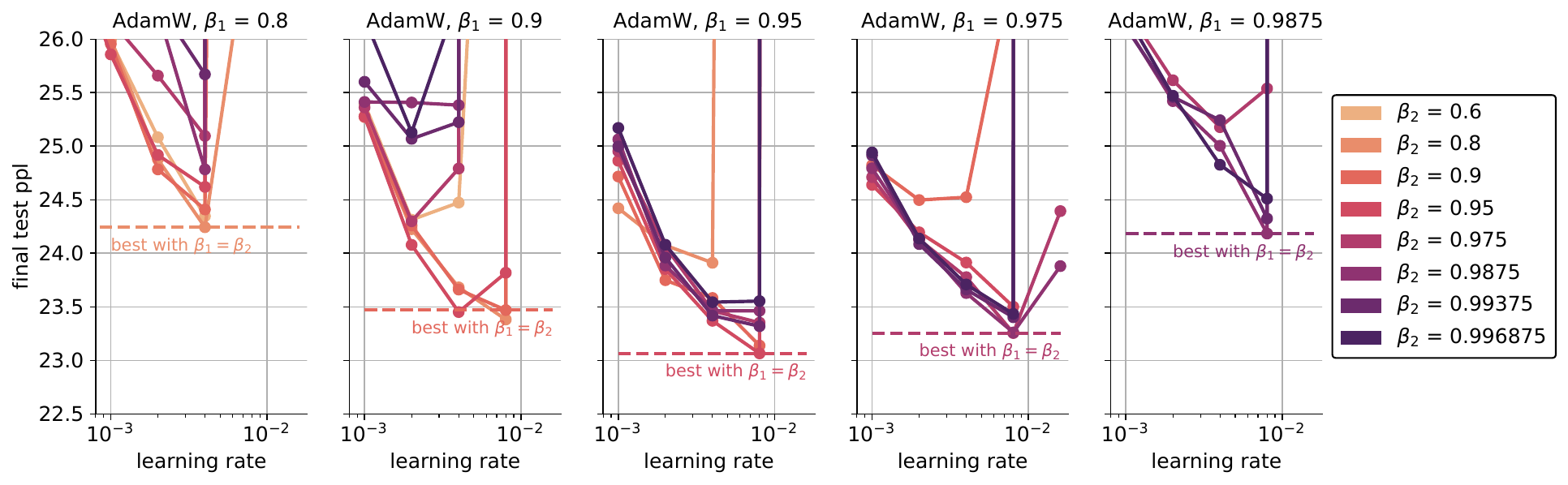}
    \caption{\small \textit{\textbf{Adam, batch size 256} trained for 2.5B tokens. Other settings are same setting as Figure~\ref{fig:big_sweep}. }}
    \label{fig:big_sweep_bs_normal}
\end{figure}

\begin{figure}[ht]
    \centering
\includegraphics[width=0.9\linewidth]{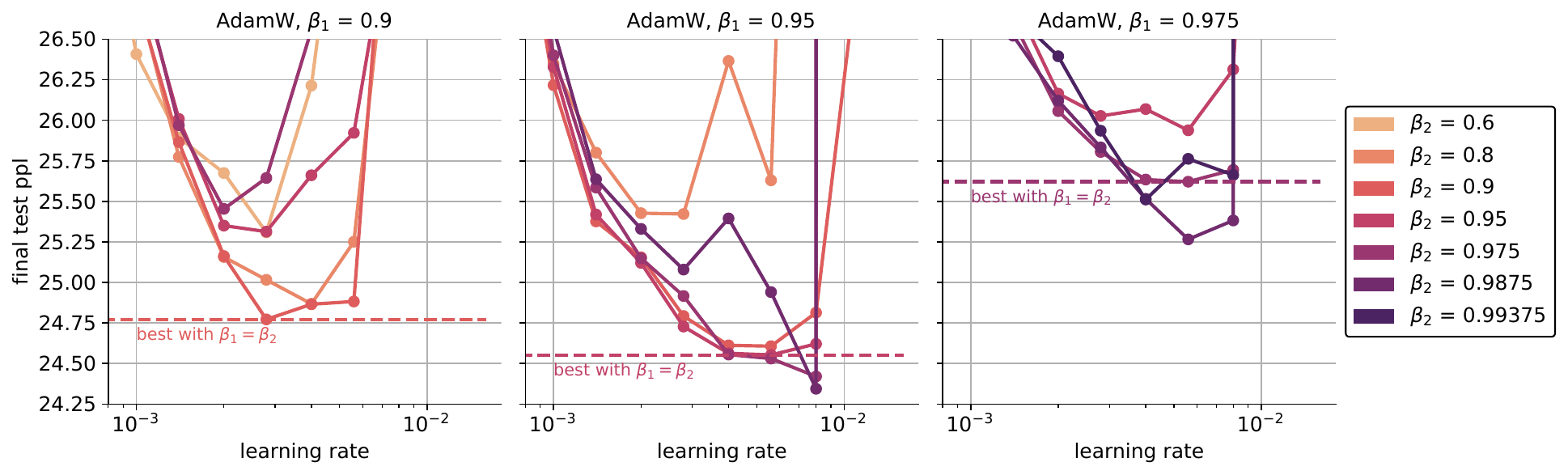}
    \caption{\small \textit{\textbf{Adam, batch size 512} trained for 2.5B tokens. Other settings are same setting as Figure~\ref{fig:big_sweep}. }}
    \label{fig:big_sweep_bs_big}
\end{figure}

\begin{figure}[ht]
    \centering
\includegraphics[width=0.9\linewidth]{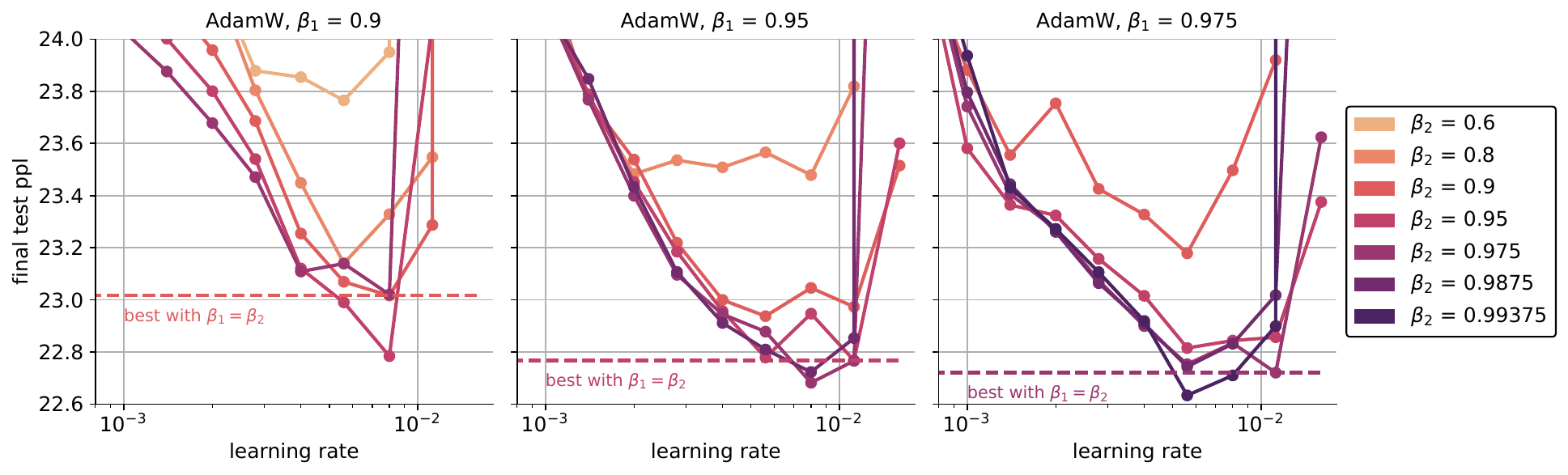}
    \caption{\small \textit{\textbf{Adam, batch size 128} trained for 2.5B tokens. Other settings are same setting as Figure~\ref{fig:big_sweep}. }}
    \label{fig:big_sweep_bs_small}
\end{figure}

\clearpage


\clearpage

\subsection{Figure~\ref{fig:big_sweep} on Fineweb~(no weight decay)}
\label{app:fineweb}

Finally, we evaluate our findings -- both strong performance of equal $\beta$s in \Adam{} and substantial gap with \Signum{} on a different dataset~(Fineweb~\citep{penedo2024the}). All other experiments in this paper are performed on SlimPajama. To add an additional axis of variation compared to previously presented settings, we here remove weight decay from all methods.

\begin{figure}[ht]
    \centering
\includegraphics[width=0.9\linewidth]{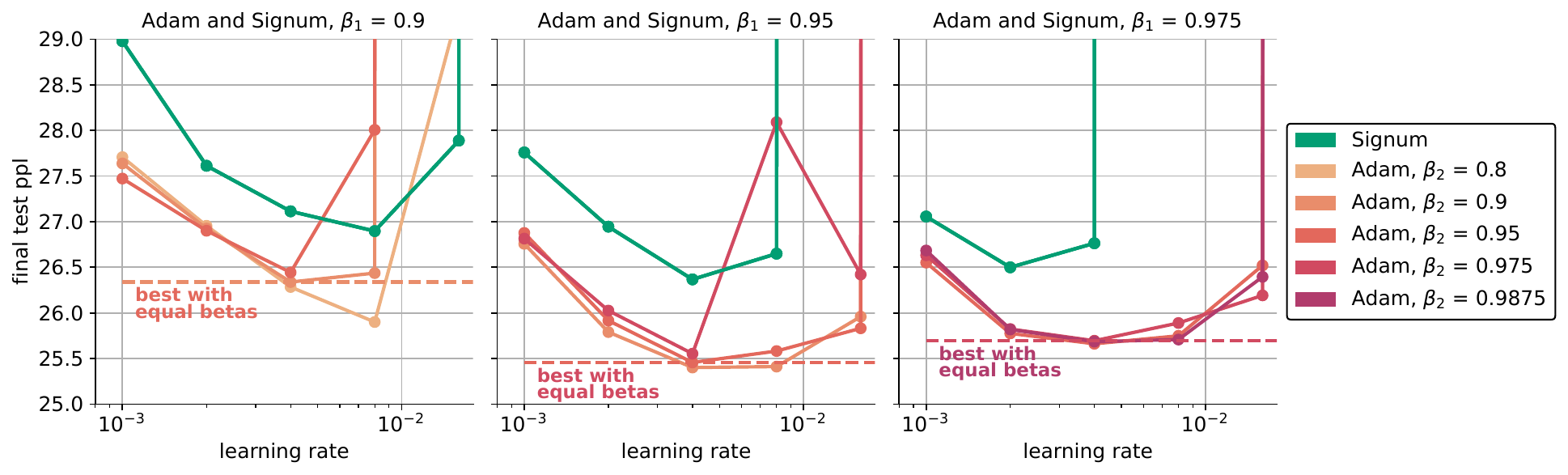}
    \caption{\small \textit{\textbf{\Adam{} and \Signum{} (no weight decay) on Fineweb}. Other settings are same as Figure~\ref{fig:big_sweep}}. For visualization purposes, here \textbf{we rescaled the visualized learning rate} of \Signum{} by a factor $\sim 10$.}
    \label{fig:160_fineweb}
\end{figure}

\subsection{Effect of Bias Correction and Zero Initialization on \Adam}
\label{app:sanity_checks}

The findings below complement our discussion in \S\ref{sec:sanity_checks}.

\begin{table}[ht]
\vspace{-3mm}
\centering
\caption{\small \textit{ZI denotes Zero init of EMA parameters, GI denotes init of EMA parameters to the measurement at first iteration, BC denotes Bias Correction. \textbf{Not} doing ZI means we initialize $m$ and $v$ at $g_0$ and $g_0^2$ respectively. Default for \Adam{} is ZI and BC. Default for Signum+WD is less clear. We found that initialization does not affect much performance in Signum, yet it does in \Adam. Performing bias correction is not as important as initialization in \Adam. All other parameters in this ablation are fixed to the optimal ones found in default settings for BC and ZI.}}
{\small
\begin{tabular}{lccccccc}
\hline
 & \Adam{} (+ZI+BC)  & \Adam{} (+ZI-BC) &  \Adam{} (+GI-BC) & Signum (+GI) &  Signum (+ZI)\\ \hline
Val ppl & 21.86\dgs{$\pm$ 0.21} & 21.89\dgs{$\pm$ 0.16} & 22.58\dgs{$\pm$ 0.35} & 23.23\dgs{$\pm$ 0.16} & 23.30\dgs{$\pm$ 0.25} \\ \hline
\end{tabular}}
\label{tab:ablation_ZI_BC}
\end{table}

\begin{figure}[ht]
    \centering
\includegraphics[width=0.88\linewidth]{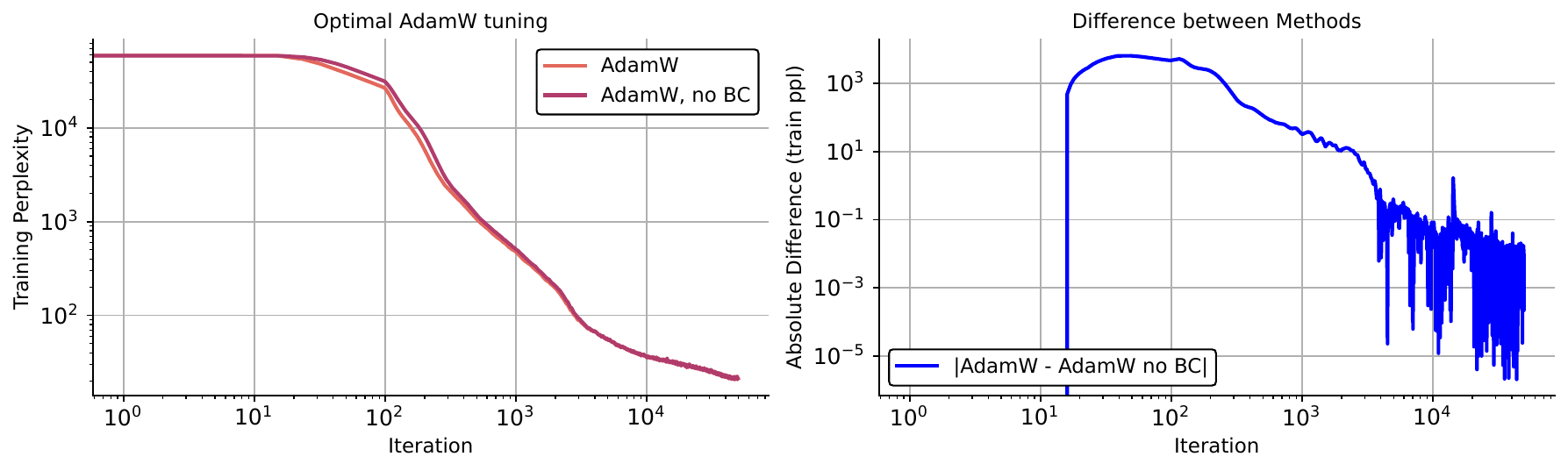}
    \caption{\small \textit{Effect of eliminating bias correction in \Adam. The difference between variants vanishes as iterations progress. Plotted is the average dynamics over 3 random seeds. }}
    \label{fig:adam_bc}
\end{figure}

\clearpage
\newpage

\section{Missing proofs and derivations}
\label{app:proofs}

\subsection{Proof of Proposition~\ref{prop:adammol} }

\propadammol*
\begin{proof}[Proof of Proposition~\ref{prop:adammol} ]
For this proof we will use the abbreviation 
\[v_k : = \texttt{EMA}_\beta[g_k^2].\]
With this abbreviation the \Adam{} update~\eqref{eq:adam-1d} can be written as
\[d_k = \frac{m_k}{\sqrt{v_k}} = \frac{m_k}{\sqrt{m_k^2 +v_k -m_k^2}}.\]
Next we will show that $v_k -m_k^2 =\beta \texttt{EMA}_\beta[(m_{k-1}-g_k)^2]$. Indeed by expanding the update of $v_{k+1}$ and $m_{k+1}$ we have that
\begin{align*}
    v_{k+1} - m_{k+1}^2 &= \beta v_k + (1-\beta) g_{k+1}^2 - (\beta m_k + (1-\beta) g_{k+1})^2\\
    &= \beta v_k + (1-\beta) g_{k+1}^2 - (\beta^2 m_k^2 + (1-\beta)^2 g_{k+1}^2 + 2\beta(1-\beta) g_{k+1} m_k)\\
    &= \beta v_k -\beta^2 m_k^2 + (1-\beta)\beta g_{k+1}^2 -2\beta(1-\beta)g_{k+1} m_k\\
    &= \beta v_k-\beta m_k^2+ \beta m_k^2-\beta^2 m_k^2 + (1-\beta)\beta g_{k+1}^2 -2\beta(1-\beta) g_{k+1} m_k\\
    &= \beta (v_k-m_k^2) + \beta(1-\beta) m_k^2 + \beta(1-\beta) g_{k+1}^2 -2\beta(1-\beta) g_{k+1} m_k\\
    &= \beta (v_k-m_k^2) + \beta(1-\beta) (m_k-g_{k+1})^2.
\end{align*}

By setting $\delta_k = v_k -m_k^2$ we have that
$$\delta_{k+1} =\beta \delta_k + \beta(1-\beta)(m_k-g_{k+1})^2 = \beta\texttt{EMA}_\beta[(m_{k-1}-g_k)^2] $$
where we used the definition of the EMA recurrence in~\eqref{eq:ema}.


\end{proof}

\subsection{Generalization of Proposition~\ref{prop:adammol} -- Necessity of equal betas for variance interpretation}
\label{sec:specific_betas}
\begin{proposition} 
\Adam{} with hyperparameters $\beta_1,\beta_2\in(0,1)$ has update of form    
$$\frac{m_k}{\sqrt{m_k^2 + \gamma \texttt{EMA}_\tau[(a m_{k-1}-b g_k)^2]}},$$
for some $a,b,\gamma\in\mathbb{R}$ and $\tau\in(0,1)$ if an only if $\beta_1=\beta_2$.
\label{prop:adamol-gen}
\end{proposition}

\begin{proof}[Proof of Proposition~\ref{prop:adamol-gen}.]
Let us expand the expression.
\begin{align*}
    v_{k+1} - m_{k+1}^2 &= \beta_2 v_k + (1-\beta_2) g_{k+1}^2 - (\beta_1 m_k + (1-\beta_1) g_{k+1})^2\\
    & = \beta_2 v_k + (1-\beta_2) g_{k+1}^2 - [\beta_1^2 m_k^2 + (1-\beta_1)^2 g_{k+1}^2 + 2\beta_1(1-\beta_1) m_k g_{k+1}]\\
    & = \beta_2 v_k - \beta_1^2 m_k^2 + [(1-\beta_2)-(1-\beta_1)^2] g_{k+1}^2 - 2\beta_1(1-\beta_1) m_k g_{k+1}
\end{align*}

\paragraph{The case of equal betas.} Notice that if $\beta_1=\beta_2=\beta$, then
\begin{align*}
    &(1-\beta)-(1-\beta)^2 = 1-\beta -(1+\beta^2-2\beta) = 1-\beta -1 -\beta^2 +2\beta = \beta(1-\beta),
\end{align*}
and so the expression gets simplified:
\begin{align*}
    v_{k+1} - m_{k+1}^2 &= \beta v_k - \beta^2 m_k^2 + \beta(1-\beta)[g_{k+1}^2 - 2 m_k g_{k+1}]
\end{align*}
Now add and subtract $\beta m_k^2$, to get

\begin{align*}
    v_{k+1} - m_{k+1}^2 &= \beta (v_k-m_{k+1}^2) + \beta(1-\beta)[m_k^2 + g_{k+1}^2 - 2 m_k g_{k+1}].
\end{align*}
where the last term is the perfect square $(m_k-g_{k+1})^2$.

\paragraph{The general setting.} One might hope for the ``stars aligning'' into a perfect square also in the general setting. For this to happen, we need to require that the term
$$[(1-\beta_2)-(1-\beta_1)^2] g_{k+1}^2 - 2\beta_1(1-\beta_1) m_k g_{k+1}$$
allows for such a simplification to happen. That is, assume to start from
$$(am_k -bg_{k+1})^2 = a^2 m_k -2ab m_k g_{k+1} + b^2 g_{k+1}^2.$$
we need 
$$b^2 = (1-\beta_2)-(1-\beta_1)^2, \quad ab = \beta_1(1-\beta_1).$$
so
$$a = \frac{\beta_1(1-\beta_1)}{\sqrt{(1-\beta_2)-(1-\beta_1)^2}}.$$
Therefore:
\begin{align*}
    &\left(\frac{\beta_1(1-\beta_1)}{\sqrt{(1-\beta_2)-(1-\beta_1)^2}} m_k - \sqrt{(1-\beta_2)-(1-\beta_1)^2}g_{k+1}\right)^2\\
    &= \frac{\beta_1^2(1-\beta_1)^2}{(1-\beta_2)-(1-\beta_1)^2} m_k^2 +[(1-\beta_2)-(1-\beta_1)^2] g_{k+1}^2 - 2\beta_1(1-\beta_1) m_k g_{k+1}
\end{align*}
Therefore, in the general setting, we can write
\begin{align*}
    v_{k+1} - m_{k+1}^2 &= \beta_2 v_k - \left(\beta_1^2 + \frac{\beta_1^2(1-\beta_1)^2}{(1-\beta_2)-(1-\beta_1)^2}\right) m_k^2  + \\
    &\qquad  +\left(\frac{\beta_1(1-\beta_1)}{\sqrt{(1-\beta_2)-(1-\beta_1)^2}} m_k - \sqrt{(1-\beta_2)-(1-\beta_1)^2}g_{k+1}\right)^2
\end{align*}
Massaging a bit, we get

\begin{align*}
    v_{k+1} - m_{k+1}^2 &= \beta_2 v_k - \frac{\beta_1^2(1-\beta_2)}{(1-\beta_2)-(1-\beta_1)^2} m_k^2  + \\
    &\qquad  +\left(\frac{\beta_1(1-\beta_1)}{\sqrt{(1-\beta_2)-(1-\beta_1)^2}} m_k - \sqrt{(1-\beta_2)-(1-\beta_1)^2}g_{k+1}\right)^2
\end{align*}

which implies

\begin{align*}
    v_{k+1} - m_{k+1}^2 &= \beta_2 \left(v_k - \frac{\beta_1^2(1-\beta_2)}{\beta_2(1-\beta_2)-\beta_2(1-\beta_1)^2} m_k^2\right)  + \\
    &\qquad  +\left(\frac{\beta_1(1-\beta_1)}{\sqrt{(1-\beta_2)-(1-\beta_1)^2}} m_k - \sqrt{(1-\beta_2)-(1-\beta_1)^2}g_{k+1}\right)^2.
\end{align*}

Therefore, the formula holds true if and only if
$$\frac{\beta_1^2(1-\beta_2)}{\beta_2(1-\beta_2)-\beta_2(1-\beta_1)^2} = 1.$$

That is, if and only if
$$\beta_1^2(1-\beta_2) = \beta_2(1-\beta_2)-\beta_2(1-\beta_1)^2.$$
The condition simplifies, as it reads:
$$\beta_1^2-\beta_1^2\beta_2 = \beta_2-\beta_2^2-\beta_2-\beta_2\beta_1^2+2\beta_1\beta_2.$$
which simplified is
$$\beta_1^2+\beta_2^2 -2\beta_1\beta_2 = 0.$$
i.e.
$$(\beta_1-\beta_2)^2=0 \quad \iff \quad \beta_1=\beta_2.$$
\end{proof}

\subsection{Proof of Theorem~\ref{thm:theoVIAdam}}
\theoVIAdam*
\begin{proof}
Recall that
\begin{align*}
    -\log p(g_{k+1} \mid m, \var) &= \frac{1}{2} \log \var + \frac{1}{2\var}(g_{k+1} - m)^2, \\
    \mathrm{KL}\left(\mathcal{N}(m_k, \var_k)\,\|\,\mathcal{N}(m, \var)\right) &=
    \frac{1}{2} \left[ \frac{\var_k}{\var} + \frac{(m_k - m)^2}{\var} - 1 - \log\left( \frac{\var_k}{\var} \right) \right].
\end{align*}

Therefore

\begin{align*}
    F(m,\sigma^2) &= -\log p(g_{k+1} \mid m, \var) + \frac{1}{\lambda}\, \mathrm{KL}\left(\mathcal{N}(m_k, \var_k)\,\|\,\mathcal{N}(m, \var)\right)\\
    &= \frac{1}{2} \log \var + \frac{1}{2\var}(g_{k+1} - m)^2 +  \frac{1}{2\lambda} \left[ \frac{\var_k}{\var} + \frac{(m_k - m)^2}{\var} - 1 - \log\left( \frac{\var_k}{\var} \right) \right]
\end{align*}
Since we are not optimizing for $\sigma^2_k$, we can replace $-\log\left( \frac{\var_k}{\var} \right) = \log(\var)$ and drop constants, gives the following objective function
\[
\min_{m, \var\geq 0}F(m, \var) =
\frac{1}{2}\frac{1+\lambda}{\lambda} \log(\var) + \frac{1}{2\var}\left[(g - m)^2 +
\frac{1}{\lambda} \left(
\var_k + (m_k - m)^2
\right) \right] + \text{const}.
\]

\textbf{Stationarity in \( m \):} Differentiating in $m$ and setting to zero gives
\[
\frac{\partial F}{\partial m} = -\frac{1}{\var}(g - m) - \frac{1}{\lambda \var}(m_k - m) = 0.
\]
Multiplying by \( \lambda \sigma^2 \), we get:
\begin{equation} \label{eq:templinzezdde}
-\lambda(g - m) - (m_k - m) = 0 \quad \Rightarrow \quad m = \frac{\lambda g + m_k}{1 + \lambda}.
\end{equation}

\textbf{Stationarity in \( \sigma^2 \):} Differentiating in $\sigma^2$ and setting to zero gives
\[
\frac{\partial F}{\partial \var} =
\frac{1}{2}  \frac{1+\lambda}{\lambda}\cdot \frac{1}{\var}
- \frac{1}{2\sigma^4} \left[ (g - m)^2 + \frac{1}{\lambda} \left( \var_k + (m_k - m)^2 \right) \right] = 0.
\]
Multiplying both sides by \( 2\sigma^4 \), and re-arranging gives:
\[
 \frac{1+\lambda}{\lambda}\var = (g - m)^2 + \frac{1}{\lambda} \left( \var_k + (m_k - m)^2 \right).
\]
Multiplying through by  $\frac{\lambda}{1+\lambda}$ gives
\begin{equation} \label{eq:tesmlefolenslfzef}
\var = \frac{\lambda(g - m)^2 +  \left[ \var_k + (m_k - m)^2 \right]}{1 + \lambda}.
\end{equation}

Now using \( m = \frac{\lambda g + m_k}{1 + \lambda} \) from~\eqref{eq:templinzezdde} we have that
\[
g - m = \frac{g - m_k}{1 + \lambda}, \quad
m_k - m = \frac{\lambda(m_k - g)}{1 + \lambda}.
\]
Therefore:
\[
(g - m)^2 = \frac{(g - m_k)^2}{(1 + \lambda)^2}, \quad
(m_k - m)^2 = \frac{\lambda^2 (g - m_k)^2}{(1 + \lambda)^2}.
\]

Using the above in the expression for \( \var \) in~\eqref{eq:tesmlefolenslfzef}, we get:
\[
\var = \frac{\lambda(g - m_k)^2}{(1 + \lambda)^2} + \frac{\var_k}{1 + \lambda}.
\]
This, together with~\eqref{eq:templinzezdde} gives the final solution
\[
\boxed{
m_{k+1} = \frac{ m_k+\lambda g }{1 + \lambda}
}
\quad \text{and} \quad
\boxed{
\var_{k+1} =  \frac{\var_k}{1 + \lambda}+\frac{\lambda(g - m_k)^2}{(1 + \lambda)^2} 
} \,.
\]
If we use the standard momentum parameterization, which corresponds to $\beta =\frac{1}{1+\lambda}$ we arrive at the stated results \eqref{eq:mom-VI-view} and \eqref{eq:varVI-view} of the theorem.
\end{proof}

\subsection{Performance of generalized Adam reformulation}
\label{sec:adavar}

As described in \S\ref{sec:balles}, we here consider performance of the update direction:

\begin{equation*}
    d_k = \frac{m_k}{\sqrt{m_k^2 + \gamma \, \EMA_\tau[(m_{k-1} - g_k)^2]}} \tag{AdaVar}
\end{equation*}

This reduces to \Adam{} with equal betas as soon as $\beta = \gamma = \tau$ but cannot be written as an \Adam{} update as soon as $\beta \ne \gamma$ or $\gamma\ne \tau$~(see proof in~\S\ref{sec:specific_betas}). Further, our theory in~\S\ref{sec:theory} shows that $\beta = \gamma = \tau$ is the only theoretically grounded choice for a precise online variational inference interpretation, also in this setting, i.e. when considering $\sigma^2_k= \gamma \, \EMA_\tau[(m_{k-1} - g_k)^2]$. We wonder if this insight correlates with optimal performance.

As one can see in Figure~\ref{fig:adavar}, we found that setting $\beta=\tau=\gamma$ leads to near optimal performance in all settings.

\begin{figure}[ht]
    \centering
    \includegraphics[height=0.3\linewidth]{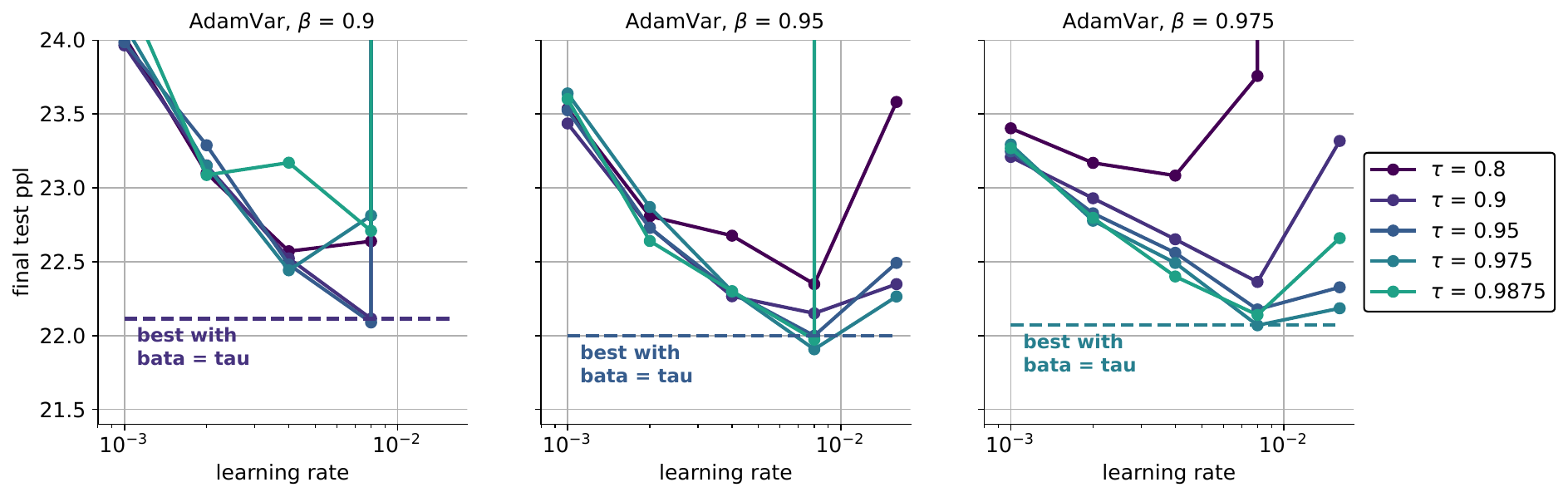}
    \includegraphics[height=0.3\linewidth]{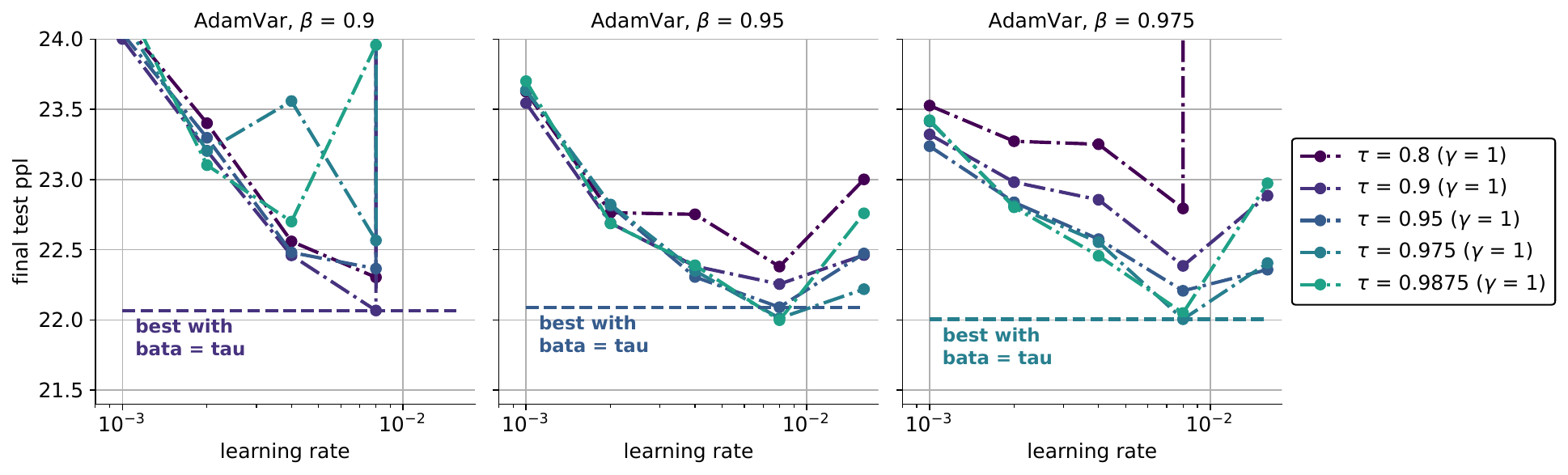}
    \caption{\textit{Performance of AdaVar \textbf{aligns with our theoretical insights}. Setup for these experiments is exactly the same as for Figure~\ref{fig:big_sweep}}.}
    \label{fig:adavar}
\end{figure}

\clearpage
\newpage
\section{Toy Quadratic Example}
\label{app:quadratic_details}
\begin{figure}[ht]
    \centering
    \includegraphics[width=0.99\linewidth]{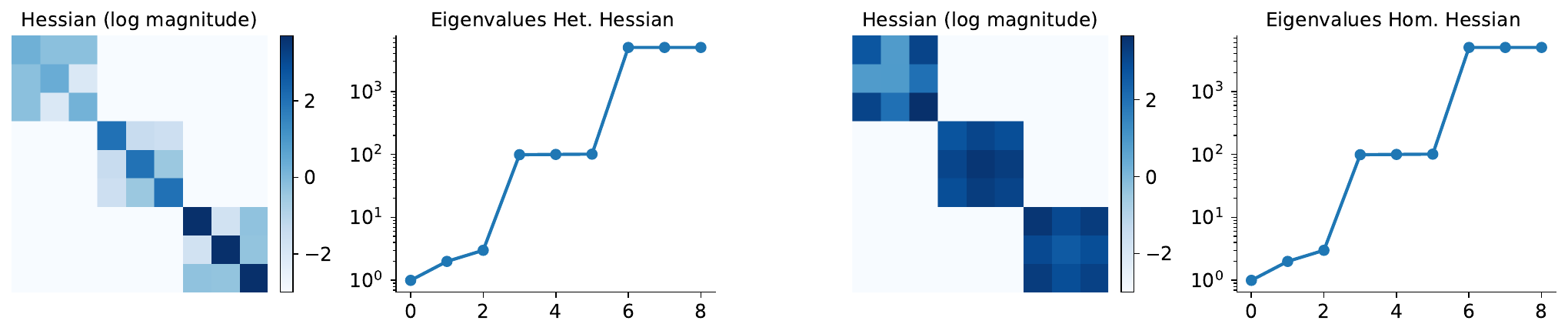}
    \caption{\textit{(left) Heterogeneous and (right) Homogeneous Hessian considered in \S\ref{sec:quadratic}.}}
    \label{fig:quadratic-hessian}
\end{figure}

Our setup here is inspired directly from the results and discussions in~\citet{zhang2024why}. Specifically, we consider the loss

$$L(w) = \frac{1}{2}w^\top H w$$

where we construct the Homogeneous and Heterogeneous Hessians using the following procedure:
\begin{itemize}
    \item We fix the eigenvalues, equal in both cases, to
    $$\text{eig}(H_{\text{hom}}) = \text{eig}(H_{\text{het}}) =\{1,2,3,99,100,101,4998,4999,5000\}.$$
    \item We choose both Hessians to be block-diagonal, with blocks of size $3\times 3$. The homogeneous Hessian has eigenvalues of different magnitude in each block, while the Heterogeneous keeps similar magnitudes in each block.
    \begin{center}
        \verb|H_details_het = [[1,2,3],[99,100,101],[4998,4999,5000]]|\\
        \verb|H_details_hom = [[1,99,4998],[2,100,4999],[3,101,5000]]|
    \end{center}
    \item For each block, we apply a random rotation to the diagonal matrix of eigenvalues, specific to each block. Each rotation is sampled from the Haar measure by decomposition of a random $3\times 3$ positive semidefinite matrix $AA^\top$, where $A\in\R^{3\times 3}$ has i.i.d. Gaussian entries.
\end{itemize}
The result is shown in Figure~\ref{fig:quadratic-hessian}.

Next, to introduce stochasticity in this setting, we simply take the square root of the Hessian to define a $9\times 9$ design matrix $X$

$$H = X^\top X, \qquad X = H^{\frac{1}{2}}$$

and subsample a number~(the batchsize) of rows of $X$ at each iteration.

\clearpage
\newpage
\section{Signal Processing Perspective}
\label{app:signal_processing}
In this last section, we examine \Adam{} through a signal processing lens, to get qualitative insights into its distinction with \Signum{} and other \SignSGD{} with momentum variants. Setting  $\beta_1=\beta_2=\beta$, we can write the \Adam{} update, without bias correction~(see \S\ref{app:sanity_checks}) as simply
$$d_k=\left(\sqrt{\texttt{EMA}_{\beta}[g_k^2]}+\epsilon\right)^{-1}\texttt{EMA}_{\beta}[g_k]$$
where $(g_k)_k$ is the gradient signal. One might wonder if this special case allows for a simpler graphical interpretatoin of \Adam{}. To do this, \textbf{we consider here fixing the gradient signal, and see how different methods process this signal.}

\paragraph{Graphical intuition.} We denote by $d_k$ the update of \Adam{} once it sees a gradient signal $(g_i)_{i\le k}$:

and plot its dynamics as a function of a synthetic one-dimensional gradient in Figure~\ref{fig:filtering}.
\begin{figure}[ht]
    \centering
    \includegraphics[width=0.99
    \linewidth]{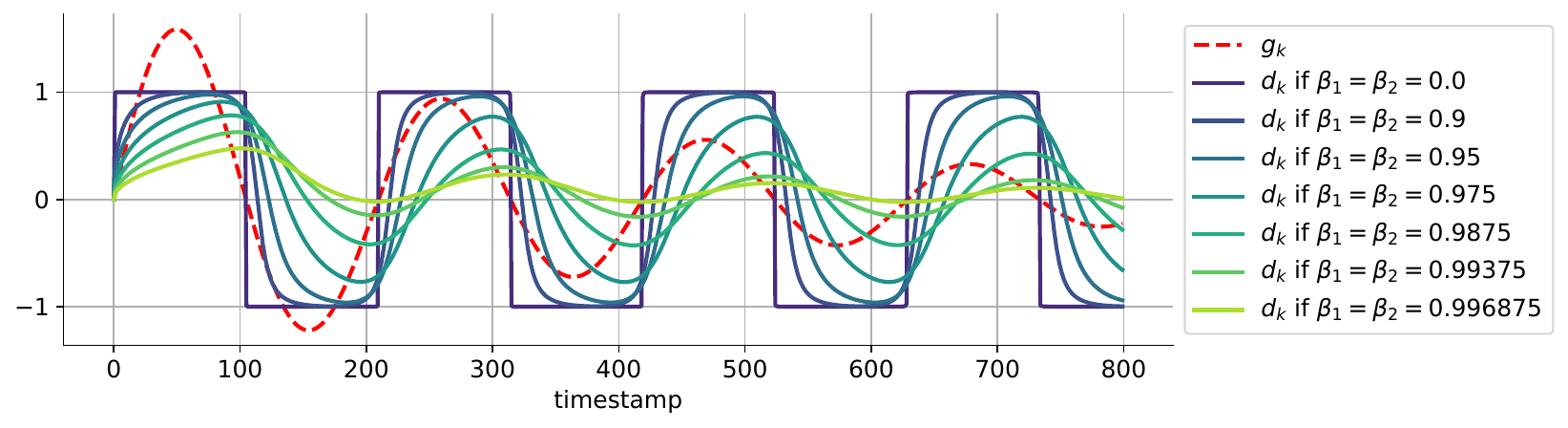}
    \caption{Filtering effect for same $\beta_1=\beta_2$.}
    \label{fig:filtering}
\end{figure}

In the example of Figure~\ref{fig:filtering}, we chose the synthetic gradient signal $$g_k = 1.8\sin(0.03 k) \exp(-0.0025 k)$$
this is a damped periodic signal {\color{red} plotted in red}. Note that this is pure filtering, there is no loss or learning process. We note the following:
\begin{enumerate}
    \item $\beta_1=\beta_2=0$ is obviously just $\texttt{sign}(g_k)$. This is plotted for comparison.
    \item For any $\beta_1=\beta_2\ne0$, $d_k$ is bounded by $1$ in magnitude. It's dynamics however, for e.g. $\beta_1=\beta_2>0$ is smooth and follows more closely the gradient, while being bounded. It is somehow a rescaled version. More on this later.
    \item Very interestingly, $d_k$ is blind to the decay term $\exp(-0.0025 k)$, the output is perfectly periodic for every $\beta_1=\beta_2$.
\end{enumerate}

Towards proceeding, note that $d_k$ \textbf{cannot be reduced to momentum on the sign or sign on the momentum(Signum)}: both variants actually destroy the signal shape, while $d_k$ maintains the shape of the original signal and has clear invariance properties. The behavior of signSGD with momentum~(2 variants) is shown in Figure~\ref{fig:filtering2}: as one can see, the behavior is drastically different from $d_k$ in Figure~\ref{fig:filtering}, an enlargement is shown in Figure~\ref{fig:filtering3}.

We now try to formalize some of the properties we observe.

\paragraph{Properties.} \Adam{} can be seen as a very special operator $T$ on gradient sequences $(g_k)_{k=0}^\infty\in \mathcal{G}\subseteq \ell_\infty$~(with normed vector space structure and notation). We can identify four distinctive properties. $T:(g_k)_{k=0}^\infty \to (d_k)_{k=0}^\infty$.
\begin{enumerate}
    \item It is \textbf{causal}.
    \item It is \textbf{invariant to positive scaling}: $T(\alpha\cdot g) = T(g)$, for any $\alpha>0$.
    \item It is \textbf{odd}: $T(-g) = -T(g)$.
    \item It has \textbf{bounded} infinity norm: $\|T(g)\|_\infty\le 1$ for all $g\in\ell_\infty$.
    \item \textbf{Density}: For any $b\in[-1,1]$ and any arbitrary $k>0$, there exists $(g_k)_{k=0}^\infty$ such that $d_k=b$.
\end{enumerate}

We are amazed by these rich set of properties, thickening our interest in better understanding the properties of \Adam{} mollification, which we study in \S\ref{sec:theory}.

\begin{figure}[ht]
    \centering
    \includegraphics[width=0.99
    \linewidth]{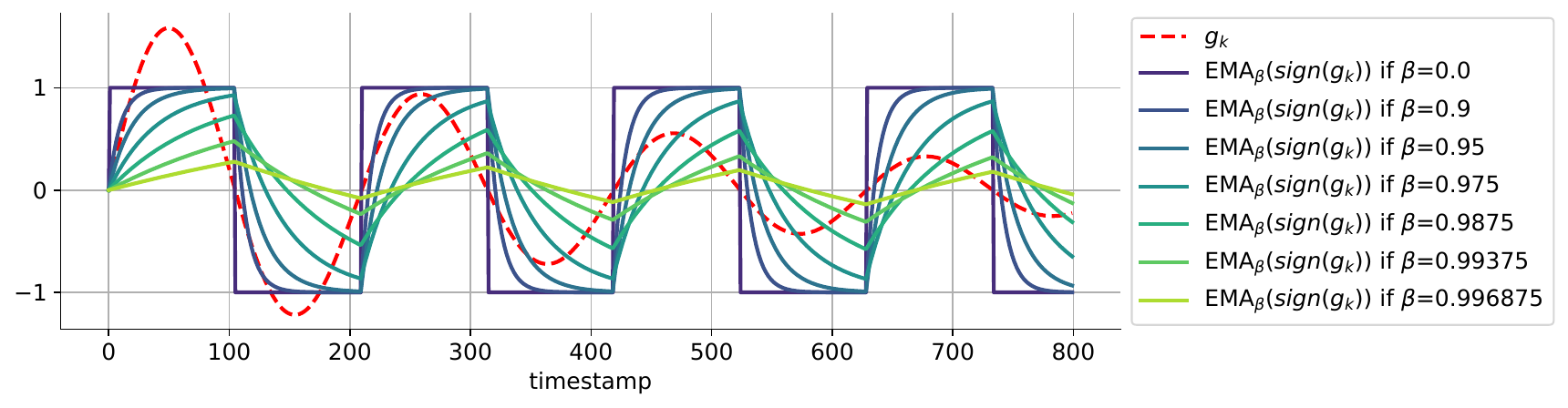}
    \includegraphics[width=0.99
    \linewidth]{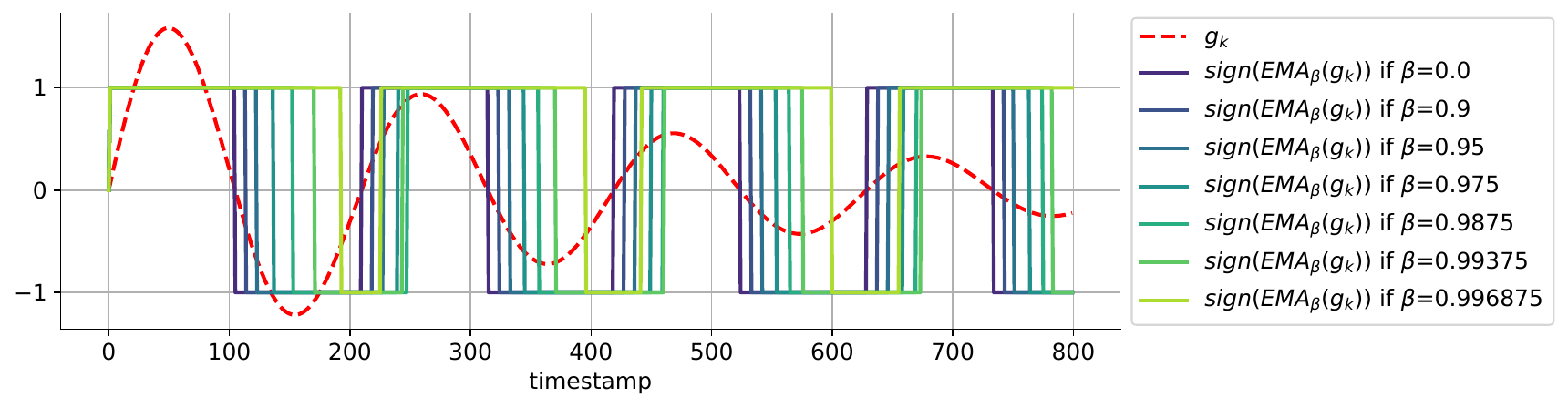}
    \caption{Filtering induced by signSGD with momentum~(2 variants, the one below is \Signum{}). Compare with Figure~\ref{fig:filtering}.}
    \label{fig:filtering2}
\end{figure}

\begin{figure}[ht]
    \centering
    \includegraphics[width=0.99
    \linewidth]{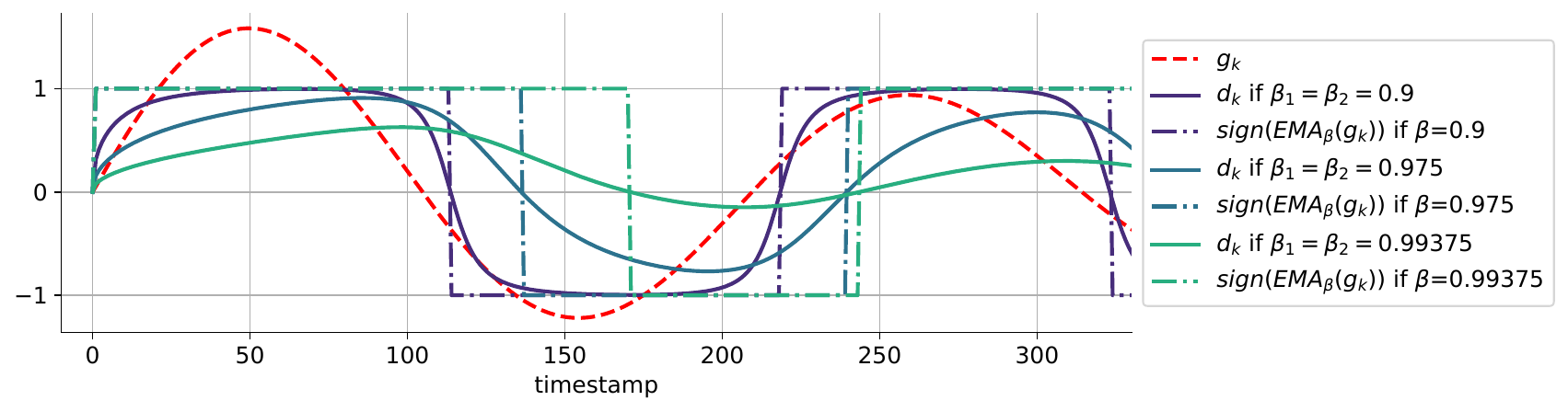}
    \caption{\Adam-like filtering compared to sign of EMA~(\Signum{}), detail.}
    \label{fig:filtering3}
\end{figure}
\vspace{15mm}
We hope this investigation ispires future effors in understanding these intriguing phenomena and properties. We conclude the paper with a quote, stolen from the Bernt Øksendal masterpiece book on SDEs:

\begin{center}
\textit{We have not succeeded in answering all our problems.\\
The answers we have found only serve to raise a whole set\\
of new questions. In some ways we feel we are as confused\\
as ever, but we believe we are confused on a higher level\\
and about more important things.\\ \ \\}

Posted outside the mathematics reading room --Tromsø University
\end{center}

\end{document}